%% file: main.tex
\documentclass[a4paper,11pt]{article}
\PassOptionsToPackage{capitalize}{cleveref}
\usepackage{style}
\usepackage[sort,square,numbers,longnamesfirst]{natbib}
\usepackage{xcolor}
\usepackage{microtype}
\usepackage{wrapfig}
\usepackage[symbol]{footmisc}

\allowdisplaybreaks

\usepackage{titling}
\title{How do simple rotations affect the implicit bias of Adam?}

\usepackage{authblk}
\author[1,*]{Adela DePavia}
\author[2,3]{Vasileios Charisopoulos}
\author[1,3,4]{Rebecca Willett}
\affil[1]{Committee on Computational and Applied Mathematics, University of Chicago}
\affil[2]{Department of Electrical \& Computer Engineering, University of Washington}
\affil[3]{NSF-Simons National Institute for Theory and Mathematics in Biology}
\affil[4]{\parbox[t]{.8\textwidth}{\centering Department of Statistics, Department of Computer Science, and Data Science Institute, University of Chicago}}
\affil[*]{Corresponding author: adepavia@uchicago.edu}

\usepackage{nicefrac}
\usepackage{algorithm}
\usepackage[noend]{algpseudocode}
\usepackage{cancel}
\usepackage[margin=1.1in]{geometry}

\makeatletter
\newcommand{\leqnomode}{\tagsleft@true}
\newcommand{\reqnomode}{\tagsleft@false}
\makeatother

\begin{document}
\maketitle

\begin{abstract}
Adaptive gradient methods such as Adam and Adagrad are widely used in machine learning, yet their effect on the generalization of learned models---relative to methods like gradient descent---remains poorly understood.
Prior work on binary classification suggests that Adam exhibits a ``richness bias,'' which can help it learn nonlinear decision boundaries closer to the Bayes-optimal decision boundary relative to gradient descent.
However, the coordinate-wise preconditioning scheme employed by Adam renders the overall method sensitive to orthogonal transformations of feature space. We show that this sensitivity can manifest as a reversal of Adam's competitive advantage: even small rotations of the underlying data distribution can make Adam forfeit its richness bias and converge to a linear decision boundary that is farther from the Bayes-optimal decision boundary than the one learned by gradient descent. 
To alleviate this issue, we show that a recently proposed reparameterization method---which applies an orthogonal transformation to the optimization objective---endows any first-order method with equivariance to data rotations, and we empirically demonstrate its ability to restore Adam's bias towards rich decision boundaries.
\end{abstract}

\input{src/intro}

\input{src/related_works}
\input{src/under_data_rotation}
\input{src/EGOP_equivariance}

\input{src/empirical_results}
\input{src/conclusions}

\section*{Acknowledgments}

We gratefully acknowledge the support of AFOSR FA9550-18-1-0166, NSF DMS-2023109, DOE DE-SC0022232, the NSF-Simons National Institute for Theory and Mathematics in Biology (NITMB) through NSF (DMS-2235451) and Simons Foundation (MP-TMPS-00005320), and the Margot and Tom Pritzker Foundation.

\bibliographystyle{plainnat}
\bibliography{ref}

\pagebreak

\appendix
\input{src/optimization_algos}
\input{src/deferred_proofs}
\input{src/Shampoo}

\end{document}

%% file: src/intro.tex
\section{Introduction}

Adaptive gradient methods---including \adam, \adagrad, and variants like \adamw---are extremely popular optimizers in neural network training and other areas of modern machine learning \cite{Crew_2020}.
Recent work has explored the impact of the choice of optimizer on generalization. For instance, \citet{vasudeva2025rich} ask, ``What is the implicit bias of \adam for nonlinear models such as [neural networks], and how does it differ from the implicit bias of (stochastic) gradient descent?" Answering this question can provide insights into the ramifications of optimizer choice in machine learning. 

The central result of \citet{vasudeva2025rich} is essentially that adaptive methods can overcome the ``simplicity bias'' of gradient descent (\GD) \cite{morwani2023simplicity, shah2020pitfalls, perez2019deep}. To illustrate this, they consider a binary classification setting, construct a family of data distributions, and analyze training a two-layer ReLU network classifier. For their family of distributions, the Bayes decision boundary is nonlinear, but training using \GD provably yields a linear classifier. In contrast, training using an adaptive optimization method yields a nonlinear decision boundary similar to the Bayes decision boundary, as illustrated in \cref{fig:OG_gamma=0}.

At first glance, these findings may suggest that adaptive methods naturally prefer ``richer'' predictors than gradient descent. However,
most algorithms in this family (including \adam) are essentially preconditioned variants of
gradient descent; crucially, their preconditioners are \emph{diagonal}---i.e., applied coordinate-wise---meaning that the overall method is not equivariant to simple transformations (e.g., rotations) of the model parameters. As we highlight in this work, rotations in data space induce rotations in parameter space for feedforward neural networks. Thus, the sensitivity of adaptive algorithms to parameter rotations implies sensitivity to data rotations. Consequently, it is natural to ask:

\begin{center}
    \begin{minipage}{0.6\textwidth}
    \centering \itshape
        How do simple transformations of the data distribution affect the implicit bias of adaptive gradient methods?
    \end{minipage}
\end{center}

In this paper, we study this question in the stylized setting of \citet{vasudeva2025rich}, considering rotations of the feature space. We show that even for small rotations of the same distribution considered in \cite{vasudeva2025rich}, adaptive optimization yields a linear classifier which, for almost all rotation angles, is further from the Bayes decision boundary than the result of \GD training, as illustrated in \cref{fig:OG_gamma=pi_over_32}. Our results complement the work of \citet{vasudeva2025rich}: taken together, these results establish that arbitrarily small perturbations of the training data can significantly worsen the generalization of adaptive gradient methods.

This dramatic shift in qualitative behavior prompts the following question: can we \textbf{algorithmically} identify a coordinate system in which adaptive methods produce rich decision boundaries?
To answer this, we turn our attention to a simple reparameterization
method proposed by the current authors~\cite{depavia2025fasteradaptiveoptimizationexpected},
which corresponds to an orthonormal transformation of model parameters derived from the
\textit{expected gradient outer product} (EGOP) matrix of the training loss. In particular,
we show that this reparameterization method coupled with adaptive optimization results in a rotation
invariance that restores the implicit bias identified by~\citet{vasudeva2025rich}, as illustrated
in~\cref{fig:EGOP_gamma=pi_over_32}. Our analysis establishing rotation invariance under EGOP-reparameterization extends to related algorithms, including \SOAP \cite{vyas2024soap} and \Shampoo \cite{gupta2018shampoo}.

\paragraph{Organization.} In the rest of this section, we introduce the notation used
throughout our paper. In Section~\ref{ssec:results-summary} we provide brief formal statements of the main results of the paper and review related literature in~\cref{ssec:related-works}.
In Section~\ref{sec:decision-boundaries-under-rotation}, we analyze the decision boundaries learned by adaptive algorithms on binary classification problems, emphasizing their sensitivity to data rotations. In Section~\ref{sec:EGOP-reparameterization}, we show that EGOP-reparameterization endows adaptive algorithms with equivariance to data rotations, producing generalization behavior that is invariant to these orthogonal transformations. Finally, in~\cref{sec:empirical_results}, we present a suite of numerical experiments illustrating
our theoretical predictions, which remain valid even in settings not explicitly covered by our theory. 

\subsection{Notation and background}
We write $\ip{\cdot, \cdot}$ for the Euclidean inner product on $\R^d$ and $\norm{\cdot} = \sqrt{\ip{\cdot, \cdot}}$ for the induced Euclidean norm. For a matrix $A \in \R^{m\times n}$, we denote its (column-major) vectorization by $\operatorname{vec}(A) \in \R^{mn}$ and its adjoint operation by
$\reshape:\R^{mn}\rightarrow \R^{m\times n}$, such that $\reshape(\vecop(A)) = A$.
We denote the $i^{\text{th}}$ entry of a vector $\theta$ by $\theta_i$ and the $k^{\text{th}}$
row of a matrix $A \in \R^{m \times n}$ by $A_{k,:}\in \R^{n}$. We use $\odot$ to denote the Hadamard (element-wise) product operation between vectors and matrices of compatible dimensions and $\otimes$ for the Kronecker product. We write $\diag(v)$ for the square matrix which equals $v$ across its
diagonal and $0$ elsewhere.
We use $\mathbb{I}_d \in \R^{d\times d}$ to denote the identity matrix. We let $O(d)$ denote the set of orthogonal $d\times d$ matrices. Finally, we write $\indic(\cdot)$ for the indicator function.

In \cref{app:methods}, we formally define the methods that we study---namely, \adam (\cref{alg:adam}),
\SignGD (\cref{alg:signgd}), gradient descent (\cref{alg:SGD}), and EGOP-reparameterization (\cref{alg:meta-algorithm-block}); they all assume first-order access to the loss function $\loss$. 

\section{Overview of main results}
\label{ssec:results-summary}
In this section, we introduce the problem setup and summarize our main results. Our goal is to study the implicit bias of \adam (\cref{alg:adam})
and gradient descent (\GD, \cref{alg:SGD}) for the synthetic binary classification
task introduced by~\citet{vasudeva2025rich} under rotations to the input data distribution.
As in~\citet{vasudeva2025rich}, we focus on 2-layer neural networks with
ReLU activations and
\emph{fixed} outer layer weights:
\begin{equation}\label{eq:og-two-layer-objective}
    f(W; x) \defeq a^\T \ReLU\big(W x\big),
\end{equation}
where $a \in \R^m$ is a fixed vector of outer layer weights, $m$ is the hidden layer width,
$\ReLU(x) = \max(0, x)$ is applied elementwise, and $W \in \R^{m \times d}$ 
are the hidden layer weights.

While we are ultimately interested in the implicit bias of \adam, we analyze a simplified,
analytically tractable version known as \SignGD (\cref{alg:signgd}) --- a memoryless instantiation of \adam ($\beta_1 = \beta_2 = 0$)
which reduces to the following update when $\epsilon = 0$:
\begin{equation}
    \label{eq:def-signGD}
    \theta_{t+1} = \theta_{t} - \eta \sign(\grad \loss(\theta_t)), \qquad \text{where }
    \sign(z) := \begin{pmatrix}
        \sign(z_i)
    \end{pmatrix}_{i=1}^p.
\end{equation}
From an analytical standpoint, \SignGD is a tractable proxy for \adam because the setting
$\beta_1 = \beta_2 = 0$ removes the dependence on trajectory history. However, as shown in
\cref{fig:opener} and throughout our empirical results, \SignGD and \adam exhibit very similar
behavior for the task at hand; consequently, our theoretical conclusions for \SignGD are highly
predictive of the behavior of both algorithms.

As our goal is to study the impact of data rotations on the implicit bias of each method, we briefly
introduce some relevant notation.
Fix a joint distribution $\mathcal{D}$ over feature-label pairs $(x, y) \in \R^{d} \times \set{\pm 1}$ and an orthogonal matrix $U \in O(d)$, 
which induces rotations/reflections about the origin via the map $x \mapsto U x$. We let $\mathcal{D}^{(U)}$ denote the transformation of $\mathcal{D}$ by $U$:

\begin{figure}[t]
     \centering
     \begin{subfigure}[t]{0.31\textwidth}
         \centering
         \includegraphics[width=\linewidth]{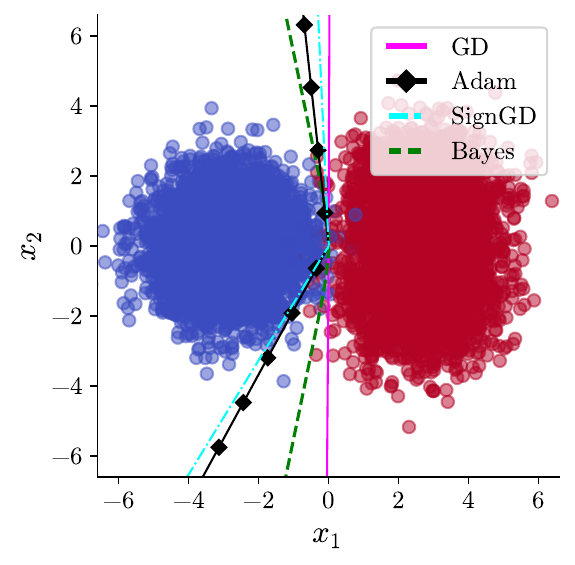}
         \caption{\centering Data rotation $\gamma=0$, learning with base algorithms}
         \label{fig:OG_gamma=0}
     \end{subfigure}
     \begin{subfigure}[t]{0.31\textwidth}
         \centering
         \includegraphics[width=\linewidth]{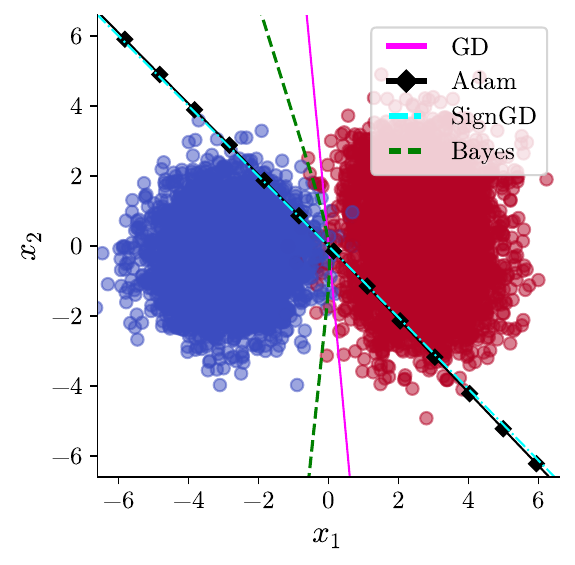}
         \caption{\centering Data rotation  $\gamma = \nicefrac{\pi}{32}$, learning with base algorithms}
         \label{fig:OG_gamma=pi_over_32}
     \end{subfigure}
     \begin{subfigure}[t]{0.35\textwidth}
         \centering
         \includegraphics[width=\linewidth]{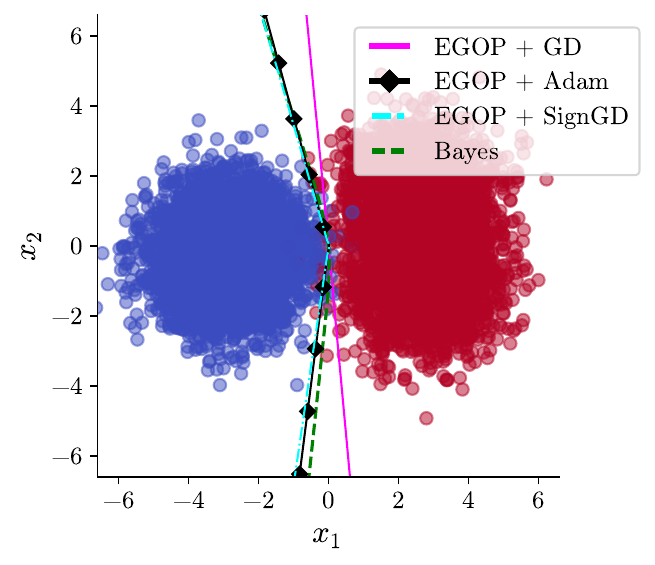}
         \caption{\centering Data rotation $\gamma = \nicefrac{\pi}{32}$, learning with EGOP reparameterization}
         \label{fig:EGOP_gamma=pi_over_32}
     \end{subfigure}
     \caption{Illustration of main results. Small data rotations can drastically change  generalization. (a) Base algorithms \adam and \SignGD produce nonlinear decision boundaries, similar to the nonlinear Bayes-optimal rule, compared with GD's linear decision boundary. (b) After rotating the data by a small angle ($\gamma = \nicefrac{\pi}{32}$ radians), the base \adam and \SignGD algorithms produce linear boundaries with poor generalization. (c) Training  EGOP-reparameterized \adam and \SignGD restores \adam's implicit bias towards nonlinear decision boundaries.
     }
     \label{fig:opener}
\end{figure}

\begin{equation}\label{eq:def-rotated-dist}
    \mathcal{D}^{(U)}\big((x,y)\big) \defeq \mathcal{D}\big((U^\T x, y)\big).
\end{equation}
Our first result shows that there are binary classification instances for which small rotations of
the data distribution can drastically change the behavior of adaptive algorithms.
\begin{theorem}[Informal; see~\cref{thm:general-rotation-linear-boundary}]
\label{thm:informal-rotation-boundary}
    For any nontrivial rotation matrix $U$, there exists a joint distribution $\mathcal{D}$ over feature-label pairs such that the following hold:
    \begin{enumerate}
        \item \label{item:informal:nonlinear} 
        \citep[Theorem 2]{vasudeva2025rich}
        Training a 2-layer ReLU network~\eqref{eq:og-two-layer-objective}
        using \SignGD with $\epsilon = 0$ to classify points from $\mathcal{D}$ produces a \emph{nonlinear} decision boundary;
        \item \label{item:informal:linear} In constrast, training the same network using the same algorithm to classify points from $\mathcal{D}^{(U)}$ produces the following \emph{linear} decision boundary:
        \[
            \set*{x \in \R^d \mid x_1 = -x_2}.
        \]
    \end{enumerate}
\end{theorem}
\Cref{item:informal:nonlinear} in~\cref{thm:general-rotation-linear-boundary} is precisely
Theorem 2 from~\citet{vasudeva2025rich}, which shows a separation between the implicit bias of
\adam and gradient descent: namely, that
\begin{displayquote}
\itshape
    ``\GD exhibits a simplicity bias, resulting in a linear decision boundary with a suboptimal margin, whereas \adam leads to much richer and more diverse features, producing a nonlinear boundary that is closer to the Bayes optimal predictor.''
\end{displayquote}
In turn, \cref{item:informal:linear} in~\cref{thm:general-rotation-linear-boundary} complements the findings of~\citet{vasudeva2025rich} by showing that
for any---potentially very small---rotation, there exist binary classification instances wherein
using \adam to learn a ReLU classifier from the rotated input data $\mathcal{D}^{(U)}$ not only learns a \emph{linear} decision boundary, but the boundary itself possesses \emph{worse} generalization properties than the linear boundary learned by gradient descent (which is invariant to rotations).
Taken together,~\cref{thm:informal-rotation-boundary} and~\cite[Theorem 2]{vasudeva2025rich} imply 
that the implicit bias of adaptive algorithms is highly sensitive to even small perturbations of the training data, as illustrated by comparing Figure~\ref{fig:OG_gamma=0} and Figure~\ref{fig:OG_gamma=pi_over_32}.

Shifting our focus towards algorithmic remedies for this sensitivity, we proceed
to show how to systematically ``reorient'' the objective in data-driven manner by means of a suitable transformation in parameter (rather than data) space.
The method, dubbed \emph{EGOP-reparameterization}, amounts to an orthonormal change of basis derived from the eigenvectors
of the \emph{expected gradient outer product} (EGOP) matrix.
\begin{definition}[EGOP]
    \label{def:egop-matrix}
    Consider a differentiable loss function $\loss: \R^p \to \R$ and a sampling distribution
    $\rho$. Then the \emph{EGOP matrix} of $\loss$ with respect to $\rho$ is defined as:
    \begin{equation}
        \label{eq:egop-def}
        \mathrm{EGOP}_{\rho}(\loss) \defeq
        \mathbb{E}_{\theta \sim \rho}[\grad \loss(\theta) \grad \loss(\theta)^{\T}].
    \end{equation}
\end{definition}
The reparameterization method, proposed in~\cite{depavia2025fasteradaptiveoptimizationexpected} as
a preprocessing step to accelerate the convergence of adaptive optimization methods,
uses the eigenvector matrix $V \in O(p)$ of $\mathrm{EGOP}_{\rho}$ to define an linearly transformed\footnote{Since the learnable parameter in~\eqref{eq:og-two-layer-objective} is matrix-valued, we calculate the EGOP using gradients with respect to the \emph{flattened} decision variable $\theta \defeq\vecop(W)$.
}
objective
$\widetilde{\loss}(\theta) \defeq (\loss \circ V)(\theta)$
which is then minimized with a first-order method such as \adagrad or \adam. We show that
EGOP-reparameterization \emph{provably} yields rotation-equivariant decision boundaries for \emph{any} first-order method with deterministic updates. For simplicity, we state this result for the binary classification problem introduced above, but as discussed in Section~\ref{sec:EGOP-reparameterization} our results hold for broad classes of objectives, data distributions, and multilayer neural networks.

\begin{theorem}[Informal; see~\cref{thm:EGOP-invariant-decision-boundaries-2d-gamma}]
\label{thm:informal-invariant}
    Consider any deterministic, first-order algorithm $\mathcal{A}$, initial point $\theta_0 \in \mathbb{R}^p$, rotation angle $\gamma \in \R$, corresponding rotation matrix $U(\gamma)$, and data distribution $\mathcal{D}$.
    Let $\theta_{\mathcal{A}}$ denote
    the result of optimizing the linear correlation loss~\eqref{eq:population_loss} with $\mathcal{A}$ starting at $\theta_{0}$ for the
    2-layer ReLU network defined in~\eqref{eq:og-two-layer-objective}, and let
    $\theta^{(\gamma)}_{\mathcal{A}}$ denote the result of optimizing the EGOP-reparameterized objective with $\mathcal{A}$ when the data have
    been rotated by $\gamma$ radians.
    Then under EGOP-reparameterization, $\mathcal{A}$ learns an equivariant decision boundary:
    \begin{equation*}
        \left\{x\in \R^{d} \ \Big|\  f(\theta^{(\gamma)}_{\mathcal{A}}; x) = 0\right\} = \left\{U(\gamma) x \ \Big|\   f(\theta_{\mathcal{A}}; x) = 0\right\}.
    \end{equation*}
\end{theorem}
Figure~\ref{fig:EGOP-array-of-angles} illustrates Theorem~\ref{thm:informal-invariant}. For simplicity, we present Theorem~\ref{thm:informal-invariant} in terms of data \textit{distributions} $\mathcal{D}$, corresponding to optimizing the population loss during training, but we note that an analogous result can be established for finite datasets.

We finally turn our attention towards a short numerical study, whose key aim is to determine whether
the theoretical predictions of Theorem~\ref{thm:informal-rotation-boundary} are still valid when several simplifying assumptions are no longer in force; namely, we examine both \SignGD and \adam (i.e., $\beta_1, \beta_2 \neq 0$), replace all expectations with finite-sample estimates, allow for $\epsilon > 0$, 
incorporate bias terms in the
hidden network layer, and make \emph{all} neural network parameters trainable (see~\cref{sec:empirical_results}). Our numerical results suggest that our core observations about the interplay between data rotations and implicit bias are consistent across configurations.

\begin{figure}[t]
     \centering
     \begin{subfigure}{0.315\textwidth}
         \centering
         \includegraphics[width=\linewidth]{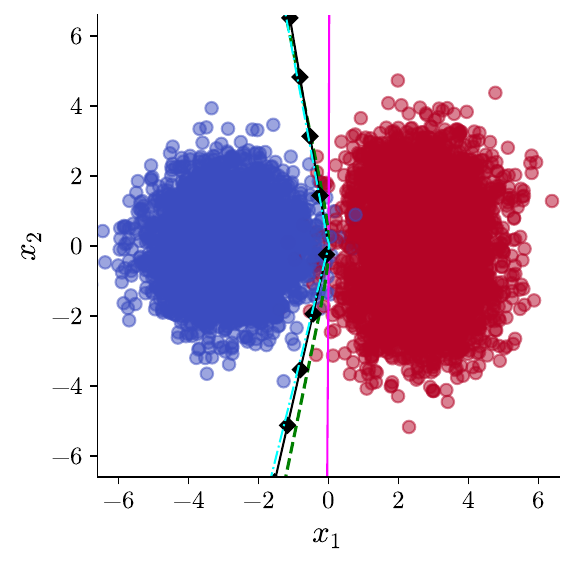}
         \caption{$\gamma = 0$}
         \label{fig:EGOP_gamma=0}
     \end{subfigure}
     \begin{subfigure}{0.35\textwidth}
         \centering
         \includegraphics[width=\linewidth]{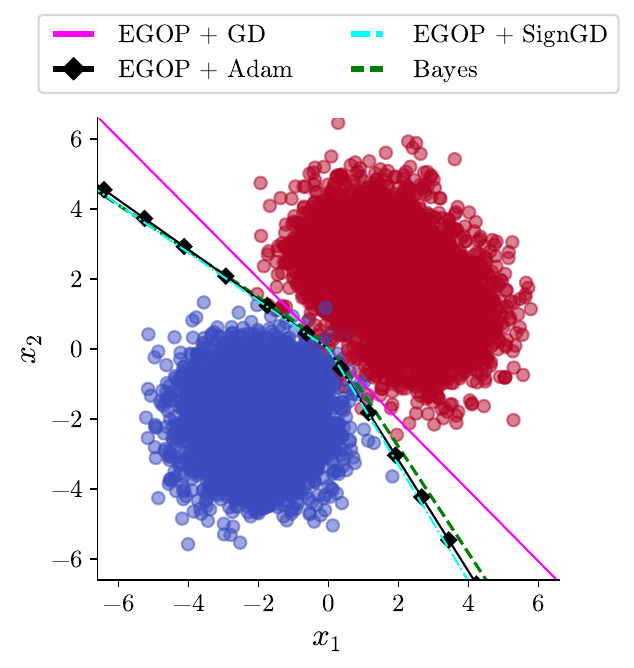}
         \caption{$\gamma = \frac{\pi}{4}$}
         \label{fig:EGOP_gamma=pi_over_4}
     \end{subfigure}
     \begin{subfigure}{0.315\textwidth}
         \centering
         \includegraphics[width=\linewidth]{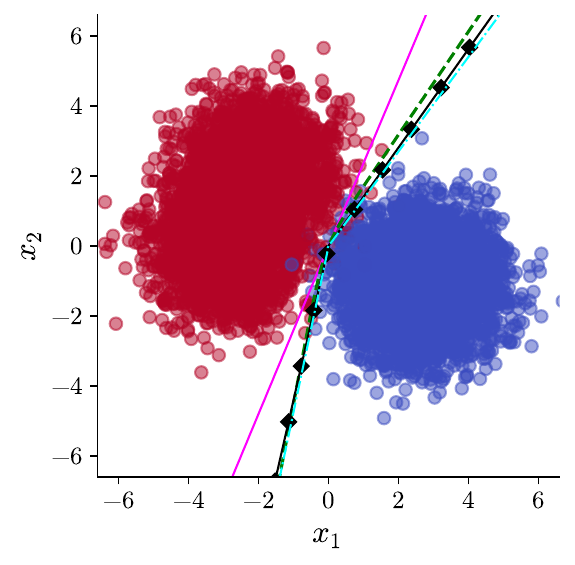}
         \caption{$\gamma = \frac{7\pi}{8}$}\label{fig:EGOP_gamma=7pi_over_8}
     \end{subfigure}
     \caption{Illustration of Theorem~\ref{thm:informal-invariant}. Decision boundaries produced by EGOP-reparameterized \adam and \SignGD are equivariant. As the data distribution rotates by varying angles $\gamma$, the decision boundaries learned by these algorithms also rotate by $\gamma$ radians. For all data rotations, EGOP-reparameterized \adam and \SignGD learn nonlinear decision boundaries that are closer to the Bayes-optimal decision boundary than those produced by gradient descent. Experiment details in Section~\ref{sec:empirical_results}.}
     \label{fig:EGOP-array-of-angles}
\end{figure}

%% file: src/related_works.tex
\section{Related works}\label{ssec:related-works}

\paragraph{Implicit bias of \adam versus \SGD.}
Existing work examines the structure of solutions found by \SGD and adaptive methods and their effect on generalization; for example, the work of~\citet{wilson2017marginal} constructs linearly separable binary classification instances where in \adam generalizes poorly relative to gradient descent. \citet{zhang2024implicit} show that \adam---equipped with a suitable learning rate---converges to a maximum $\ell_{\infty}$ margin solution, while \SGD instead maximizes the
$\ell_{2}$ margin~\cite{soudry2018implicit}.
This area of research has produced the \textit{simplicity bias hypothesis}, which posits that gradient descent exhibits a preference for ``simpler'' (often linear) predictors~\cite{perez2019deep, shah2020pitfalls}, which may act as a type of regularization and thus enable generalization on overparameterized problems;
however, this simplicity bias
can also lead to poor generalization when training data contains spurious correlations~\cite{arpit2017closer,shah2020pitfalls,morwani2023simplicity,kalimeris2019sgd}. Closest to ours is the work of~\citet{vasudeva2025rich}, which draws a separation between
\SGD and \adam through the lens of simplicity bias for a stylized binary classification problem. Our paper complements this stream of research by showing that the bias of \adam towards ``richer'' models is contingent on the geometry of the underlying data distribution and not an intrinsic feature of the algorithm.

\paragraph{Rotations in parameter space.}

Several recent works study how rotations in parameter space impact the performance of adaptive optimization algorithms \cite{depavia2025fasteradaptiveoptimizationexpected, maes2024understanding, vyas2024soap,gupta2018shampoo}. \citet{depavia2025fasteradaptiveoptimizationexpected} propose a data-driven method for computing an orthonormal reparameterization, called \textit{EGOP-reparameterization}, and prove that this method improves the worst-case convergence guarantees of adaptive algorithms for various machine learning problems. \citet{maes2024understanding} show that random rotations in parameter space can degrade the performance of adaptive algorithms in training neural networks, and together with~\citet{vyas2024soap,zhao2024galore} propose reparameterization methods to improve them. Our work draws an explicit connection between rotations in data space and 
reparameterization methods (in parameter space), and demonstrates that reparameterization
methods can temper the effect of coordinate system choices on the behavior of algorithms
like \adam.

\paragraph{Coordinate-wise analyses of adaptive algorithms} Motivated by the lack of rotational equvariance, another stream of research studies adaptive algorithms through connections to the $\ell_\infty$ geometry of objective functions \cite{xie2024adam,xie2024implicit,zhang2024implicit}.
Therein, the $\ell_{\infty}$ norm has been identified as an implicit regularizer for \adamw under suitable learning rate schedules~\cite{xie2024implicit}, as well as a non-Euclidean measure of smoothness establishing a separation between the convergence rates of adaptive methods and \SGD~\cite{xie2024adam,liu2024adagrad,jiang2024convergence} that was not achievable with Euclidean
analyses. Our work complements this stream of research by studying how non-invariance impacts
generalization, rather than convergence rates. \citet{zhang2024implicit} study the implicit bias of Adam on linearly separable data and show that Adam converges towards a linear classifier with maximum $\ell_{\infty}$ margin.
In contrast, our work considers non-separable instances where the Bayes-optimal decision boundary is \emph{nonlinear}, showing that the generalization capability of \adam is highly sensitive to data rotations
in this setting.

%% file: src/under_data_rotation.tex
\section{Decision boundaries under data rotation}\label{sec:decision-boundaries-under-rotation}
In this section, we present our main theoretical result on the implicit bias of
\SignGD under rotation. We start by formalizing the problem setup.

\subsection{Problem setup}\label{sec:problem-setup}
\paragraph{Data distribution.} We first define the data distribution $\mathcal{D}$ over samples $(x,y) \in \R^d \times \set{\pm 1}$ used in~\cite{vasudeva2025rich}. For simplicity, we restrict our attention to distributions in $\R^2$. Samples
from $\mathcal{D}$ are generated according to the following Gaussian mixture:
\begin{equation}\label{eq:data-distribution}
    \left\{ \; \begin{aligned}
    y &\sim \operatorname{Unif}(\{\pm 1\}), \\
    \epsilon &\sim \operatorname{Unif}(\{\pm 1\}),\\
    x_1 &\sim \mathcal{N}\Big(\frac{\mu_1 - \mu_3}{2} + y\frac{\mu_1 + \mu_3}{2}, \sigma^2\Big), \\
    x_2 &\sim \mathcal{N}\Big(\mu_2\epsilon\cdot\frac{y+1}{2}, \sigma^2\Big)
    \end{aligned} \right.
\end{equation}
\Cref{fig:OG_gamma=0} shows $(x,y)$ pairs drawn from this distribution in dimension $d=2$ 
with $\omega = 2(1+\sqrt{2})$, $\mu = 1.15$, and $\sigma = 1.0$. Points are marked red if $y=+1$ and blue if $y=-1$.

We study the problem of training 2-layer ReLU networks belonging to the class defined in \cref{eq:og-two-layer-objective} to predict the label $y$ of a datapoint $x$ drawn from this mixture of Gaussians. The family of models defined in \cref{eq:og-two-layer-objective} does not contain bias terms, and thus can only express decision boundaries which pass through the origin. We make the following \textit{realizability assumption} on the data distribution to ensure that the Bayes-optimal predictor for this dataset also passes through the origin:
\begin{assumption}\label{assumption:realizability}
    The parameters $\mu_1, \mu_2$ and $\mu_3$ in~\eqref{eq:data-distribution} satisfy
    \[
        \mu_1 = \frac{\mu_2}{2} \left(\omega - \frac{1}{\omega}\right) \quad
        \text{and} \quad
        \mu_3 = \frac{\mu_2}{2} \left(\omega + \frac{1}{\omega}\right),
    \]
    where $\omega \defeq \frac{\mu_1 + \mu_3}{\mu_2}$ is a slope parameter. We also
    write $\mu \defeq \mu_2$ for brevity.
\end{assumption}

Intuitively, $\omega$ corresponds to the slope (and negative slope) of the approximately piecewise-linear components of the Bayes-optimal predictor decision boundary, and $\mu$ corresponds to the difference between the means of the clusters with label $+1$. 

Since most of our analysis focuses on the special case $d = 2$, wherein proper rotations
can be identified with scalar values, we use $\gamma$ to denote angles of rotation and let $U(\gamma)$ denote the following rotation matrix:
\begin{equation}\label{eq:def-U-gamma}
    U(\gamma) \defeq \begin{bmatrix*}[r]
        \cos(\gamma) & -\sin(\gamma)\\
        \sin(\gamma) & \cos(\gamma)
    \end{bmatrix*}.
\end{equation}

The transformation $z \mapsto U(\gamma) z$ rotates any 2D vector counter-clockwise by $\gamma$ radians. Accordingly, we write $\mathcal{D}^{(\gamma)}$ for the rotated distribution:
\begin{equation}\label{eq:def-P_gamma}
    \mathcal{D}^{(\gamma)}\big((x,y)\big) \defeq \mathcal{D}\Big((\Ugamma^{\T} x, y)\Big)
\end{equation}
Informally, samples $(x_\gamma,y_\gamma)$ drawn from $\mathcal{D}^{(\gamma)}$ have the same distribution over labels $y$ as samples from $\mathcal{D}$, but compared to samples from $\mathcal{D}$ the data $x_\gamma$ is rotated by $\gamma$ degrees counter-clockwise. \Cref{fig:OG_gamma=pi_over_4} illustrates samples drawn from $\mathcal{D}^{(\pi/4)}$.

\paragraph{Training objective.}
Following \citet{vasudeva2025rich}, we focus on the
\textit{expected linear correlation loss}:
\begin{equation}\label{eq:population_loss}
    \loss^{(\gamma)}(\theta; f) \defeq 
    \mathbb{E}_{(x, y) \sim \mathcal{D}^{(\gamma)}}[
    -y \cdot f(\theta; x)
    ].
\end{equation}
In order to assess the performance of different algorithms, we compare the decision boundaries they produce to the decision boundary of the \textit{Bayes-optimal predictor}, which is defined as
\begin{equation}\label{eq:def-bayes-opt}
    \BayesOpt(\cdot) \defeq \argmax_{f\colon \mathbb{R}^d \to \{\pm 1\}}\mathbb{P}_{(x,y)\sim \mathcal{D}^{(\gamma)}}[f(x) = y \mid x]
\end{equation}
and achieves the minimum classification risk. For \emph{any} distribution $\mathcal{D}$
satisfying \cref{assumption:realizability}, we can explicitly characterize the Bayes-optimal predictor for the setting $\gamma = 0$:
\begin{equation}\label{eq:bayes-opt-our-problem}
    \BayesOpt[0](x) = \begin{cases}
        +1 \quad &\text{if}\quad x_2 \geq -\omega x_1 -\frac{\sigma^2}{\mu} \log \left(\frac{1+\exp(-2\mu x_2/\sigma^2)}{2}\right)\\
        -1 &\text{else.}
    \end{cases}
\end{equation}
The Bayes classifier is equivariant to data rotations: for general $\gamma > 0$, $\BayesOpt$  is a rotation of~\eqref{eq:bayes-opt-our-problem} by
$\gamma$ radians. For the case $\gamma=0$, \eqref{eq:bayes-opt-our-problem} describes a nonlinear decision boundary passing through the origin that approaches a piecewise linear curve in the limit as the signal to noise ratio $\nicefrac{\mu}{\sigma} \to \infty$; \cref{fig:opener} visualizes the
Bayes-optimal decision boundary in dashed green.

\subsection{Rotations change the implicit bias of \SignGD}\label{ssec:learning-with-rotated-data}
We now study the impact of rotations on the decision boundary
learned by \SignGD (\cref{alg:signgd}). The flavor of our result is the following: for
\emph{any} rotation angle {$\gamma \in (0,2\pi)\setminus \set{\nicefrac{k\pi}{2}}_{k \in \mathbb{N}}$}, there exists a slope $\omega$
and parameters $\mu, \sigma$ inducing a distribution $\mathcal{D}^{(\gamma)}$ such that
minimizing~\eqref{eq:population_loss} with \SignGD leads to a \emph{linear} decision
boundary. 
Setting the stage, we quantify the range of parameters under which our results hold:
\begin{assumption}\label{assumption:general-rotation-omega-mu-sigma-relationship}
    Let $S$ denote the signal-to-noise ratio $S := \nicefrac{\mu}{\sigma}$. For a fixed rotation angle $\gamma \in (0,\pi/4]$ and slope parameter $\omega$ satisfying
    \begin{equation}\label{eq:variable-gamma-omega-assumption}
        \omega >\sqrt{2/\cmu}
        \quad \text{for}\quad \cmu \defeq
        \Phi\left(S \cdot \min\set*{
            1, 
            \frac{\omega - \nicefrac{1}{\omega}}{2}
        }\right)
        \in \left[\frac{1}{2}, 1\right]
    \end{equation}
    where $\Phi(\cdot)$ denotes the CDF of the standard Gaussian distribution, we assume $S \in (0, \infty)$ satisfies
    \begin{align}\label{eq:variable-gamma-full-assumption}
       \frac{1}{S} \leq \sqrt{\frac{\pi}{2}}\left( 
       \frac{1}{2}\left(\cmu\omega-\frac{2}{\omega}\right)\sin(\gamma)-\cos(\gamma)\right).
    \end{align}
\end{assumption}
Given a fixed slope parameter $\omega$, Assumption~\ref{assumption:general-rotation-omega-mu-sigma-relationship} requires that $1/S$ must be sufficiently small. As $\gamma\rightarrow 0$, $\sin(\gamma)\rightarrow 0$ and $\cos(\gamma)\rightarrow 1$, so Assumption~\ref{assumption:general-rotation-omega-mu-sigma-relationship} becomes more restrictive for small $\gamma$. We note that \citet{vasudeva2025rich} make related assumptions in their setting, which corresponds to $\gamma=0$: they require Assumption~\ref{assumption:realizability} to
hold, as well as 
\[
    \omega \geq 1+\sqrt{2}\approx 2.41, \quad \frac{2}{3} \leq \frac{1}{S} \leq \frac{5}{4}.
\]
In \cref{fig:comparing-decision-boundaries}, we visualize results for $\omega = 2(1+\sqrt{2})$, $\sigma = 1$, and $\mu = 1.15$; these parameter settings simultaneously satisfy the assumptions of \citet{vasudeva2025rich} and Assumption~\ref{assumption:general-rotation-omega-mu-sigma-relationship} for the rotation $\gamma = \pi/4$.

With~\cref{assumption:general-rotation-omega-mu-sigma-relationship} in place, we characterize the classifier found by $\SignGD$ for $\gamma > 0$. In the following, we show this for $\gamma \in (0,\pi/4]$ for simplicity of presentation, but the analysis extends trivially to other rotations {$\gamma \in (0,2\pi)\setminus \set{\nicefrac{k\pi}{2}}_{k \in \mathbb{N}}$}.

\begin{theorem}\label{thm:general-rotation-linear-boundary}
    Fix $\gamma\in (0,\pi/4]$ and consider $f$ as in~\eqref{eq:og-two-layer-objective}
    with hidden size $m$ and fixed outer layer weights $a \in \{\pm 1\}^m$. Let $W_{t}$ denote the $t^{\text{th}}$ iterate produced
    by \SignGD with $\epsilon = 0$ and any $\eta > 0$ applied to~\eqref{eq:population_loss}. Then if~\cref{assumption:realizability,assumption:general-rotation-omega-mu-sigma-relationship} are in force, we have that
    \begin{equation}
        \lim_{t \to \infty} \frac{(W_t)_{k, :}}{\norm{(W_t)_{k, :}}}
        = 
            \frac{a_k}{\sqrt{2}} [
                1, 1
            ]^{\T},
        \quad \text{for all $k \in \set{1, \dots, m}$.}
        \label{eq:normalized-solution-at-infinity}
    \end{equation}
    In particular, the solution described in~\eqref{eq:normalized-solution-at-infinity}
    yields the following classification rule
    \begin{equation}
        \lim_{t \to \infty} \sign(f(W_{t}; x)) = \sign(x_1 + x_2),
    \end{equation}
    whose decision boundary is equal to the subspace $\set{x \mid x_1 = -x_2}$.
\end{theorem}

\begin{figure}[t]
     \begin{subfigure}[c]{0.28\textwidth}
         \centering
         \includegraphics[width=\linewidth]{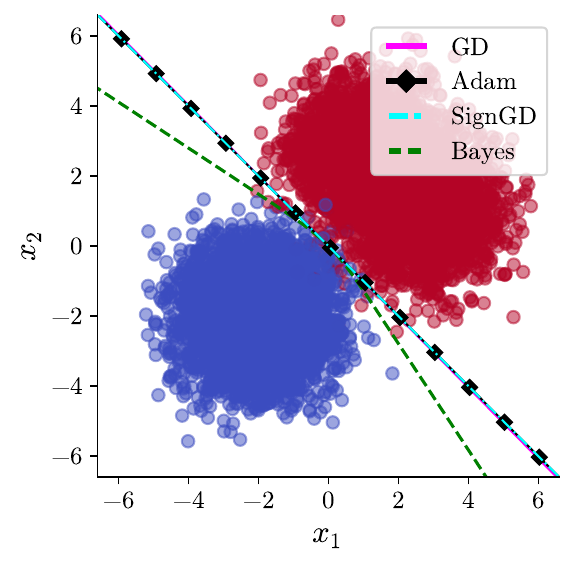}
         \caption{}
         \label{fig:OG_gamma=pi_over_4}
     \end{subfigure}
     \begin{subfigure}[c]{0.72\textwidth}
         \centering
         \includegraphics[width=\linewidth]{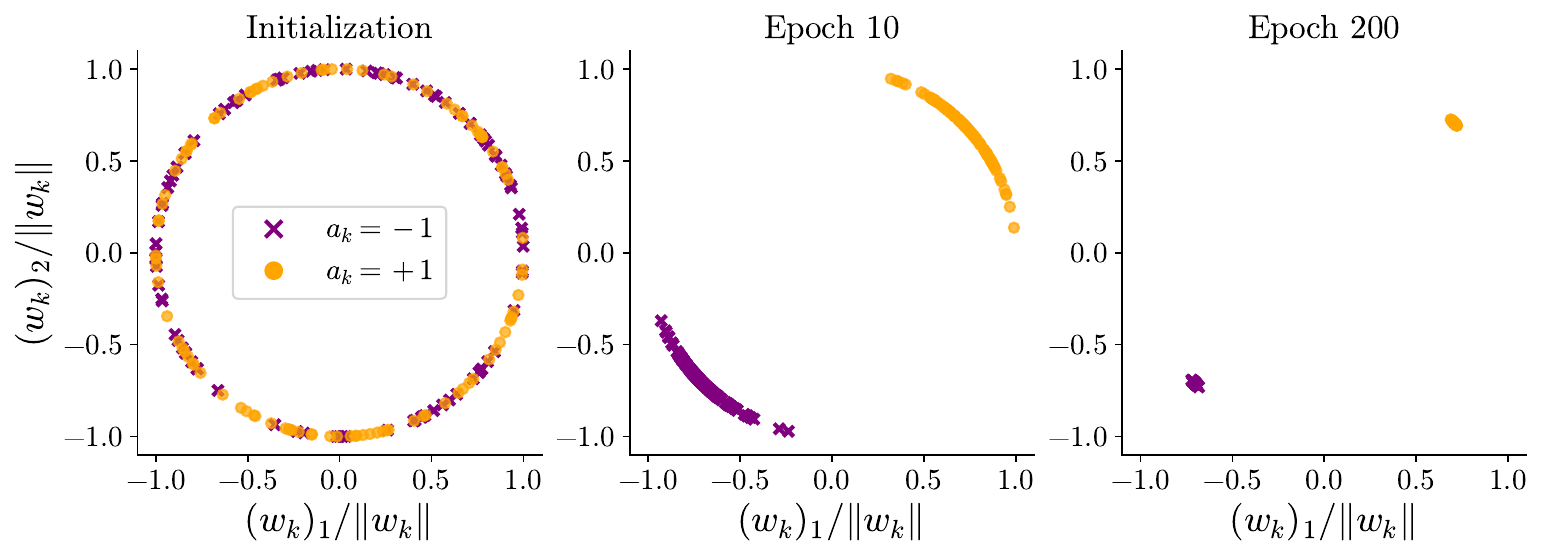}
         \caption{}
         \label{fig:learned_weights}
     \end{subfigure}
     \caption{
     (a) Illustration of Theorem~\ref{thm:general-rotation-linear-boundary}. When $\gamma = \pi/4$, \adam and \SignGD learn a linear decision boundary identical to that learned by gradient descent. Compare with \cref{fig:OG_gamma=0}, which shows that when $\gamma=0$, the decision boundaries produced by \adam and \SignGD are nonlinear and are closer to the Bayes' optimal decision boundary than that learned by GD. (b) For each epoch $t$ pictured, we visualize the weights produced by training with \adam. For every $k\in \{1,\dots,m\}$, we scatter $\nicefrac{w_k}{\norm{w_k}}$ where $w_k \defeq (W_t)_{k,:}$ denotes the $k$th row of weights. As predicted by Theorem~\ref{thm:general-rotation-linear-boundary}, these values rapidly converge to $a_k \cdot \nicefrac{[1, 1]^\T}{ \sqrt{2}}$. Details in Section~\ref{sec:empirical_results}.
     }
     \label{fig:comparing-decision-boundaries}
\end{figure}

Some comments are in order. First, we note that~\cref{thm:general-rotation-linear-boundary} states that \emph{for any rotation} $\gamma \in (0, \pi/4]$, there exist
binary classification instances such that \SignGD converges \emph{to the same linear
decision boundary} $x_1 = -x_2$, independent of $\gamma$. This contrasts with the decision boundary learned by \GD, which is linear but equivariant and thus rotates with the data distribution (see \cref{sec:invariant-opt}).
Second, while~\cref{thm:general-rotation-linear-boundary} holds for the {expected} linear correlation loss, \cref{fig:comparing-decision-boundaries} suggests a similar phenomenon occurs in the finite-sample setting.

Theorem~\ref{thm:general-rotation-linear-boundary} is a direct consequence of Lemma~\ref{lemma:update-for-any-deg-rotation}, which gives a precise characterization of \SignGD updates.
\begin{lemma}\label{lemma:update-for-any-deg-rotation}
    Fix $\gamma \in (0,\pi/4]$, and let~\cref{assumption:realizability,assumption:general-rotation-omega-mu-sigma-relationship} hold. Then for any $W \in \R^{m \times 2}$ the gradient of loss $\loss^{\rot}(\cdot)$ from~\eqref{eq:population_loss} satisfies
    \[
        \sign(\grad_{W_{k, :}} \loss^{\rot}(W; f)) = -a_k \cdot [1, 1]^{\T}.
    \]
\end{lemma}
For any $k\in [m]$ and any initialization $W_0$, Lemma~\ref{lemma:update-for-any-deg-rotation} implies that the $k^{\text{th}}$ row of the weight matrix after $t$ updates
of \SignGD with stepsize $\eta > 0$ satisfies
\[
    (W_t)_{k,:} = (W_0)_{k,:} - \eta \sum_{\tau = 1}^{t-1}
    \sign(\grad_{W_{k,:}} \loss^{\rot}(W_{\tau}; f)) =
    (W_0)_{k,:} + \eta (t-1) a_k \cdot \begin{bmatrix*}[r] 1 \\ 1 \end{bmatrix*}.
\]
Thus in the limit as $t\rightarrow \infty$, we obtain exactly the conclusion
of~\cref{thm:general-rotation-linear-boundary}:
\[
    \lim_{t\rightarrow \infty}\frac{(W_{t})_{k, :}}{\norm{(W_{t})_{k, :}}}
  = \frac{a_k [1, 1]^{\T}}{\norm{a_k [1, 1]^{\T}}} =
    \frac{a_k}{\sqrt{2}} [1, 1]^{\T}.
\]
Figure~\ref{fig:learned_weights} shows that empirically, these are precisely the weights learned during training.

%% file: src/EGOP_equivariance.tex
\section{Implicit bias under EGOP-reparameterization}\label{sec:EGOP-reparameterization}
Invariance of generalization to simple transformations of the input is a desirable feature of learning algorithms,
since the generalization of invariant algorithms is less likely to be sensitive to small perturbations of the data.
In this section, we first give a precise characterization of such invariance with respect to data rotations, using \GD as an example. We then show that \textit{EGOP-reparameterization} makes the generalization of adaptive algorithms invariant to data rotations. For simplicity, in this section we present results in terms of the data distributions and loss functions defined in Section~\ref{sec:decision-boundaries-under-rotation}, but in Section~\ref{ssec:appendix-equivariance} we generalize these results to broad families of data distributions and loss functions, including higher-dimensional settings. 

In this section, we consider an expressive class of neural network architectures, which include as special cases the 2-layer ReLU networks studied in Section~\ref{sec:decision-boundaries-under-rotation}.

\begin{definition}\label{def:multilayer-network}
    A family of feed-forward\footnote{These results can be generalized to other architectures, including convolutional neural networks and residual neural networks with linear projection shortcuts.} networks 
    with $L$ hidden layers, $f:\R^{p}\times \R^d \rightarrow \R$ , is
    \begin{equation}\label{eq:multilayer-network}
    f(\theta; x) \defeq W_{L+1} h_{L}(\theta; x) + b_{L+1} \hspace{.3cm} \text{for}\hspace{.2cm}\begin{cases}
        h_{j}(\theta; x) \defeq \act_{j}(W_{j-1} h_{j-1}(\theta; x) + b_{j-1}) \ \forall j\in \{2,...,L\}\\
        h_1(\theta; x) \defeq x
    \end{cases}
    \end{equation}
    where $\act_{j}(\cdot)$ denotes the activation function for layer $j$, $W_j$ denotes the weights for layer $j$ of dimensions
    \[
        W_1 \in \R^{m_1 \times d}\quad \text{and}\quad W_j \in \R^{m_j \times m_{j-1}} \ \forall j \in \{2,\dots,L+1\},
    \]
    $b_j$ denotes the bias vector for layer $j$ of dimension $b_{j} \in \R^{m_{j}}$, and
    \[
        \theta \defeq \left[\vecop(W_{1}),\dots, \vecop{(W_{L+1})}, b_1,\dots, b_{L+1}\right] \in \R^p. 
    \]
    For a fixed choice of architecture hyperparameters, we denote the corresponding family of networks by $\mathcal{F}\left( d, L, \{m_j\}_{j=1}^{L+1}, \{\act_j\}_{j=2}^{L+1}\right)$, which we shorten to $\mathcal{F}$ where appropriate.
\end{definition}

For this general class of network architectures, \emph{rotations in data space induce rotations in parameter space}, as formalized by the following lemma, whose proof we defer to Section~\ref{ssec:appendix-equivariance}.
\begin{lemma}\label{lemma:generalized-rotations-in-data-are-rotations-in-param-gamma}
    Consider any $\mathcal{F}$ satisfying Definition~\ref{def:multilayer-network}, and any $f\in \mathcal{F}$. Consider any data rotation angle $\gamma \in \R$, and recall the corresponding orthogonal matrix $U(\gamma)$ defined in \cref{eq:def-U-gamma}. Then there exists a $Q^{(\gamma)} \in O(p)$ such that
    \[
        f(\theta; U(\gamma)x) = f(Q^{(\gamma)}\theta; x).
    \]
    See \cref{eq:def-Q-U-gamma} for the explicit construction of $Q^{(\gamma)}$.
\end{lemma}

\subsection{Equivariant algorithms: a case study with gradient descent}
\label{sec:invariant-opt}

The classic algorithms \GD and \GD with momentum are equivariant to rotations of the data distribution: fix any family of networks satisfying Definition~\ref{def:multilayer-network}, any data rotation angle $\gamma$, any number of time steps $T$, and any initialization $\theta_0 \in \R^{p}$. Let $\theta_{\GD}^{(0)}$ denote the result of minimizing $\loss^{(0)}(\cdot, \cdot)$ with \GD
initialized at $\theta_0$,
and let $\theta_{\GD}^{(\gamma)}$ denote the result
of minimizing $\loss^{(\gamma)}(\cdot, \cdot)$ with $\GD$ initialized at $(Q^{(\gamma)})^{\T} \theta_{0}$,
for $Q^{(\gamma)}$ as in~\cref{lemma:generalized-rotations-in-data-are-rotations-in-param-gamma},
for $T$ timesteps.
Then for any $\gamma\in \R$, the following hold:
\begin{enumerate}
    \item \label{item:sgd-equivalence} The iterates of \GD in both settings are equivalent up to rotation:
    \[
        \theta^{(\gamma)}_{\GD} = (Q^{(\gamma)})^\T \theta^{(0)}_{\GD}.
    \]
    \item \label{item:net-invariance}
    The resultant trained neural networks have equivariant forward maps:
    \[
        f\left(\theta^{(\gamma)}_{\GD}; x \right) =
        f\left(\theta^{(0)}_{\GD}; U(\gamma)^\T x \right),
    \]
    where $U(\gamma)$ is defined in \cref{eq:def-U-gamma}. The above immediately implies equivariance of the decision boundaries:
    \[
        \left\{x\in \R^{d} \ \Big|\  f(\theta^{(\gamma)}_{\GD}; x) = 0\right\} = \left\{U(\gamma) x \ \Big|\   f(\theta^{(0)}_{\GD}; x) = 0\right\}.
    \]
\end{enumerate}
These facts imply that the generalization of decision boundaries produced by training with \GD is \textit{invariant} to rotations of the data distribution. For example, one consequence of fact~\eqref{item:net-invariance} is that rotations to the data distribution cannot change the linearity/nonlinearity of decision boundaries produced by training with \GD. The above facts also hold for \GD with momentum.

\paragraph{Equivariance of stochastic methods} The preceding discussion focuses on \GD and \GD with momentum for simplicity.
However, similar properties hold for stochastic first-order oracles as long as the samples used to produce stochastic gradient estimates of $f(\theta^{(\gamma)}_{\GD}; \cdot)$ and $f(\theta^{(0)}_{\GD}; \cdot)$ are equivalent up to rotation by $\gamma$ degrees.

\input{src/generalizing_equivariance}

%% file: src/generalizing_equivariance.tex
\subsection{EGOP implicit bias is invariant to data rotation}

As shown in \cref{sec:decision-boundaries-under-rotation}, adaptive gradient methods do not share the equivariance of \GD articulated in \cref{sec:invariant-opt}, and are sensitive to even small rotations of the data distribution. We now show that, under EGOP-reparameterization, the implicit biases of \adam, \SignGD,
and other adaptive first-order algorithms become invariant to rotations of the underlying data distribution.

Given a target loss function $\mathcal{L}(\theta)$, EGOP-reparameterization computes the eigenvector matrix $V\in O(p)$ of the matrix  $\mathrm{EGOP}_{\rho}$, defined in ~\cref{def:egop-matrix}.
The procedure then defines a reparameterization objective $\widetilde{\loss}(\theta) \defeq (\loss \circ V)(\theta)$ which is minimized with a first-order method such as \adagrad or \adam. Pseudocode for EGOP-reparameterization is provided in \cref{alg:meta-algorithm-block}.

We assume that the EGOP sampling distribution $\rho(\cdot)$ is \textit{rotationally invariant}. Many common distributions, including scaled Gaussians, satisfy this assumption. 
\begin{assumption}\label{assumption:rotational-invariance}
    The sampling distribution $\rho(\cdot)$ in~\cref{def:egop-matrix} is rotationally-invariant: for any $\theta\in \R^p$ and any $V\in O(p)$,
    $
        \rho(\theta) = \rho(V\theta).
    $
\end{assumption}

We analyze the impact of data rotations on the solutions produced by EGOP reparameterization. Fix any family of networks satisfying Definition~\ref{def:multilayer-network}. For any rotation angle $\gamma \in \R$, recall the rotated data distribution $\mathcal{D}^{(\gamma)}$ and corresponding population loss $\loss^{(\gamma)}(\theta; f)$ defined in Section~\ref{sec:decision-boundaries-under-rotation} (cf. \cref{eq:def-P_gamma,eq:population_loss}). We write $V_{\gamma} \in O(p)$ for the eigenbasis of the EGOP matrix for the problem setting under data rotation $\gamma$:
\begin{equation}\label{eq:EGOP-eigenbasis-gamma}
    V_{\gamma} \defeq \mathtt{eigenvectors}(\mathrm{EGOP}_{\rho}(\mathcal{L}^{(\gamma)}(\cdot\,; f))).
\end{equation}

We define the EGOP-reparameterized network $f_{\gamma}(\cdot\,; \cdot)$ as the forward map parameterized by the eigenbasis $V_\gamma$: 
\begin{equation}\label{eq:def-EGOP-forward-map-gamma}
    f_{\gamma}(\theta; x) \defeq f(V_{\gamma} \theta; x).
\end{equation}
The EGOP-reparameterized network is trained by optimizing the reparameterized loss $\mathcal{L}^{(\gamma)}(\theta; f_{\gamma})$, the instantiation of \cref{eq:population_loss} with EGOP-reparameterized forward map $f_{\gamma}$.

We consider the class of \textit{deterministic, first-order} algorithms. These are algorithms which, given an objective $\mathcal{L}(\cdot)$, generate iterates according to
\begin{equation}\label{eq:deterministic-update}
    \theta_{t+1} = \mathtt{update}\left(\{\theta_\tau\}_{\tau\leq t}, \{\nabla_{\theta} \mathcal{L}(\theta_{\tau})\}_{\tau \leq t}\right),
\end{equation}
where $\mathtt{update}(\cdot,\cdot)$ is a deterministic function.

We prove that EGOP reparameterization endows such algorithms with equivariance to rotations in data space.
\begin{theorem}\label{thm:EGOP-invariant-decision-boundaries-2d-gamma}
    Consider any family of networks satisfying Definition~\ref{def:multilayer-network}, any deterministic, first-order algorithm $\mathcal{A}$, and any EGOP sampling distribution $\rho$ satisfying \cref{assumption:rotational-invariance}. For any $\gamma\in \R$, let $\theta^{(\gamma)}_T$ denote the result of optimizing the EGOP-reparameterized objective $\mathcal{L}^{(\gamma)}(\cdot; f_{\gamma})$ with $\mathcal{A}$ for $T$ timesteps, given initialization $\theta_0$. Then the resultant forward map defined in \cref{eq:def-EGOP-forward-map-gamma} is \emph{equivariant} under rotations of the data:
    \[
        f_{\gamma}\left(\theta^{(\gamma)}_t; x\right) = f_{0}\left(\theta^{(0)}_t; U(\gamma)^\T x\right)\quad   \forall t\in \{1,\dots,T\}
    \]
    where $U(\gamma)$ is as-defined in \cref{eq:def-U-gamma}. Equivariance of the forward map immediately implies equivariance of the decision boundary. Moreover, the iterates themselves are \emph{invariant} under rotations of the data distribution: for any $t\in \{1,\dots, T\}$, we have that $\theta^{(\gamma)}_t=\theta^{(0)}_t$.
\end{theorem}

Theorem~\ref{thm:EGOP-invariant-decision-boundaries} implies that the generalization of any such algorithm $\mathcal{A}$ is \textit{invariant} to rotations of the data distribution. Figure~\ref{fig:EGOP-array-of-angles} illustrates Theorem~\ref{thm:EGOP-invariant-decision-boundaries-2d-gamma}. In Section~\ref{ssec:appendix-equivariance}, we generalize Theorem~\ref{thm:EGOP-invariant-decision-boundaries-2d-gamma} and show that analogous results hold for broad families of data distributions and loss functions and in higher-dimensional settings. We note that while these results are presented for reparameterization with the exact EGOP eigenbasis, as defined in \cref{eq:EGOP-eigenbasis-gamma}, the proofs presented in Section~\ref{sec:deferred-proofs} extend to the setting when the EGOP matrix is estimated with gradient samples.

%% file: src/empirical_results.tex
\section{Empirical results in the finite data regime}\label{sec:empirical_results}

In this section, we provide details for our numerical experiments and empirically examine relaxations of the simplifying assumptions made in our analysis. For figures throughout, we draw samples from $\mathcal{D}$ (cf. \cref{eq:data-distribution}) with the following parameters:
\begin{equation}\label{eq:empirical-parameter-settings}
    \omega = 2(1+\sqrt{2}), \quad \sigma = 1.0 \quad \mu = 1.15.
\end{equation}
These values simultaneously satisfy \cref{assumption:realizability}, \cref{assumption:general-rotation-omega-mu-sigma-relationship} for $\gamma = \pi/4$, and also the assumptions necessary for the guarantees in \citet{vasudeva2025rich} to apply. For each figure, we generate $10^{4}$ data samples $(x,y)$ following \cref{eq:def-P_gamma} for corresponding value of $\gamma$. 

\begin{wrapfigure}{l}{0.4\textwidth}
    \includegraphics[width=0.9\linewidth]{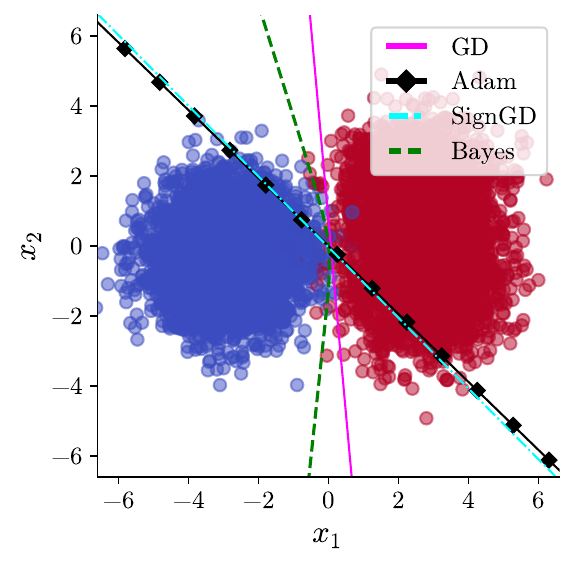}
    \captionsetup{margin=0.3cm}
    \caption{Relaxing simplifying assumptions:  for $\gamma=\nicefrac{\pi}{32}$, decision boundaries produced by a 2-layer ReLU network including bias terms and trainable outer-layer weights, defined in \eqref{eq:full-network-class}. \adam and \SignGD still yield the decision boundary predicted by Thm.~\ref{thm:general-rotation-linear-boundary}.
}\label{fig:trainA-useBias}
\end{wrapfigure}

Empirically, we find that the behavior predicted by \cref{thm:general-rotation-linear-boundary} holds even when we relax \cref{assumption:general-rotation-omega-mu-sigma-relationship}. For example, in \cref{fig:OG_gamma=pi_over_32} the parameter values in \eqref{eq:empirical-parameter-settings} do not satisfy \cref{assumption:general-rotation-omega-mu-sigma-relationship} for the rotation angle $\gamma = \pi/32$. However, the observed decision boundary is still linear and equals the theoretical prediction for admissible nonnegative values of $\gamma$.

We train 2-layer ReLU networks with hidden width $m=200$. Unless otherwise noted, we use the architecture specified in  \cref{eq:og-two-layer-objective}, which is the setting addressed by Theorem~\ref{thm:general-rotation-linear-boundary}, meaning we train only first-layer weights, do not include bias terms, and define the fixed outer-layer weights as Rademacher i.i.d. variables. 
Figure~\ref{fig:trainA-useBias} is the notable exception, where we verify that the predictions of our theory hold empirically under the relaxation of two of our simplifying assumptions: namely, the fact that \eqref{eq:og-two-layer-objective} does not contain bias terms and uses fixed outer-layer weights. In Figure~\ref{fig:trainA-useBias}, we display results obtained when training the more general 2-layer ReLU network:
\begin{equation}\label{eq:full-network-class}
    f_{\operatorname{full}}(\theta;x) = \theta_1^\T \ReLU\big(\reshape(\theta_2)x + b\big),
\end{equation}
where all parameters $\theta_1$, $\theta_2$ and $b$ are trainable. Comparing Figure~\ref{fig:trainA-useBias} with Figure~\ref{fig:OG_gamma=pi_over_32} shows that the inclusion of bias and trainable outer-layer weights does not change the core observation that under small rotations, both \adam and \SignGD learn linear decision boundaries that generalize poorly relative to the Bayes optimal decision boundary.

We conduct experiments in the finite-sample setting using the sample linear correlation loss:
\begin{equation}\label{eq:sample-loss}
    \loss_{N}(\theta)\defeq \frac{1}{N}\sum_{k=1}^N -y_k f(\theta; x_k)
\end{equation}
We train for 1000 epochs using full-batch gradients. For each algorithm (\GD, \adam, \SignGD), we use $\eta = 0.01$, though we find that our results are robust to the choice of learning rate. As stated in \cref{eq:def-signGD}, \SignGD is equivalent to \adam with parameters $\beta_1 = \beta_2 = 0$. In our experiments, we run \adam with $\beta_1 = \beta_2 = 0.9999$.

\begin{figure}[t]
     \centering
     \begin{subfigure}[t]{0.32\textwidth}
         \centering
         \includegraphics[width=\linewidth]{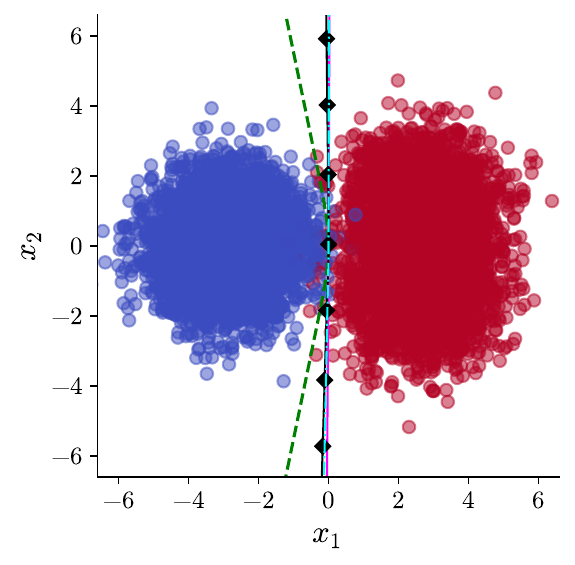}
         \caption{$\gamma = 0$}
         \label{fig:OG_eps2_gamma=0}
     \end{subfigure}
     \begin{subfigure}[t]{0.32\textwidth}
         \centering
         \includegraphics[width=\linewidth]{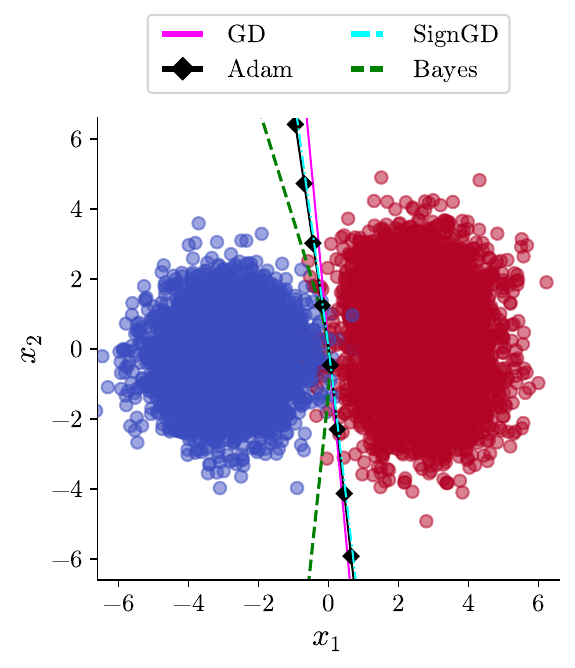}
         \caption{$\gamma = \frac{\pi}{32}$}
         \label{fig:OG_eps2_gamma=pi_over_32}
     \end{subfigure}
     \begin{subfigure}[t]{0.32\textwidth}
         \centering
         \includegraphics[width=\linewidth]{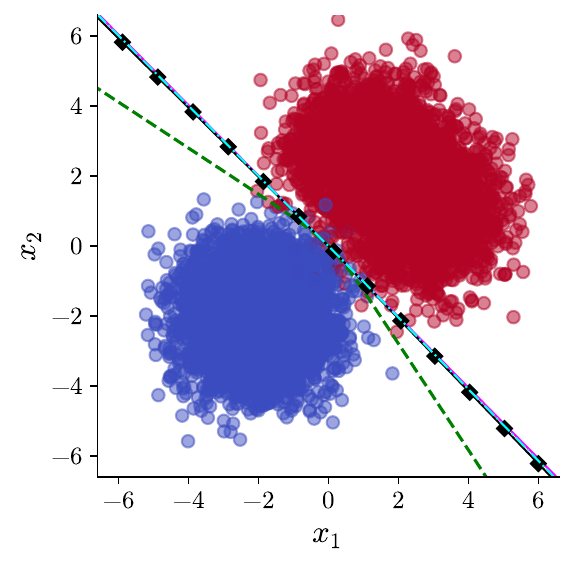}
         \caption{$\gamma = \frac{\pi}{4}$}
         \label{fig:OG_eps2_gamma=pi_over_4}
     \end{subfigure}    
     \caption{For \adam and \SignGD without reparameterization, increasing $\epsilon$ alone does \emph{not} lead to the nonlinear, good generalization, equivariant boundaries produced by EGOP-repameterization. Here we visualize results using \adam and \SignGD, with $\epsilon=3.0$, and \textit{without} reparameterization. Compare these results with those in Figures~\ref{fig:OG_gamma=0}, \ref{fig:OG_gamma=pi_over_32}, and \ref{fig:OG_gamma=pi_over_4}, where \adam and \SignGD are used \textit{without} reparameterization but with a lower value of $\epsilon$, namely $\epsilon=10^{-8}$.
     }
     \label{fig:OG_coordinates_larger_eps}
\end{figure}

\paragraph{EGOP reparameterization} EGOP reparameterization employs Monte Carlo sampling of the gradients in order to estimate the EGOP matrix, as detailed in \cref{alg:meta-algorithm-block}. We perform EGOP reparameterization using full-batch gradients of the sample loss in \cref{eq:sample-loss}. Let $M$ denote the number of gradient samples used to estimate the EGOP matrix. In our experiments, we set $M = 2m$ (where $m = 1000$ is the hidden width used in our experiments) which is exactly equal to the number of trainable parameters $p = 2m$. Thus this value of $M$ is a small number of samples for estimating the $p \times p$ EGOP matrix, though we note that our results are also robust to the choice of $M$.

When optimizing the EGOP-reparameterized objective with \adam and \SignGD, we find that moderate values of $\epsilon$ are necessary in order to produce optimal generalization. In Figure~\ref{fig:EGOP_gamma=pi_over_32} and Figure~\ref{fig:EGOP-array-of-angles}, we use $\epsilon = 3.0$ for both EGOP-reparameterized \adam and \SignGD. This value is larger than the default value for $\epsilon$:  when using \adam and \SignGD \textit{without} reparameterization, e.g., Figure~\ref{fig:OG_gamma=0}, we use $\epsilon = 10^{-8}$, which is a more typical value for this parameter. 

\emph{Both} EGOP-reparameterization \emph{and} larger values of $\epsilon$ are co-requisites for obtaining nonlinear, equivariant decision boundaries; when using \adam and \SignGD \textit{without} reparameterization, a larger value of $\epsilon$ leads to nearly-linear boundaries even in the case $\gamma = 0$, as shown in Figure~\ref{fig:OG_eps2_gamma=0}, and does not produce equivariant non-linear decision boundaries. It does, however, at least promote greater equivariance in this linear boundary, which improves generalization for some data rotations; comparing Figure~\ref{fig:OG_eps2_gamma=pi_over_32} with Figure~\ref{fig:OG_gamma=pi_over_32} shows that for small rotation angles ($\gamma = \nicefrac{\pi}{32}$) the linear boundary learned with the larger value of $\epsilon = 3.0$ is closer to that of \GD than the linear boundary learned by \adam when $\epsilon=10^{-8}$, and has better generalization.

%% file: src/conclusions.tex
\section{Conclusions and future work}

Adaptive gradient methods (such as \adam) commonly used in machine learning employ a diagonal preconditioner, rendering the overall method sensitive to small transformations of the parameter space. In machine learning, this manifests as sensitivity of the learned predictor to rotations of the training data distribution and must be carefully considered in order to fully understand the implicit bias of the underlying methods.

This paper shows that (i) the ``richness'' of decision boundaries learned by
\adam relative to gradient descent is by no means generic, but rather dependent on the input space representation; and (ii) a simple reparameterization
of the decision variable via the EGOP matrix (cf.~\cref{def:egop-matrix}) equips \emph{any} first-order
method with orthogonal invariance, and thus restores the richness bias of \SignGD / \adam that was originally established for axis-aligned data distributions. More generally, our analysis shows that transformations like EGOP-reparameterization provide a valuable mechanism for understanding the implicit bias of \adam and other adaptive optimization methods used to train machine learning models. As shown in \cref{appendix:Shampoo}, our analysis of the equivariance of EGOP-reparameterization extends to related algorithms, including \Shampoo and \SOAP \cite{gupta2018shampoo,vyas2024soap}.

This work focuses on the sensitivity of adaptive gradient methods to orthogonal transformations of the data distribution. An interesting avenue for future work would be to analyze the impact of more general data perturbations on adaptive gradient algorithms, such as nonlinear transformations or perturbations on only a subset of training points.

%% file: src/optimization_algos.tex
\section{Review of optimization methods considered}
\label{app:methods}
Below, we formally define the methods that we use---namely, \adam (\cref{alg:adam}),
\SignGD (\cref{alg:signgd}) and gradient descent (\cref{alg:SGD}); they all assume first-order access to the loss function $\loss$. Note that \SignGD is equivalent to \adam with $\beta_1 = \beta_2 = 0$.

\begin{algorithm}[h]
\caption{\adam}
\begin{algorithmic}[1]
    \State \textbf{Input}: First-order oracle $(\loss(\cdot), \grad \loss(\cdot))$, initialization $\theta_0 \in \R^p$, stepsize $\eta > 0$, $\beta_1, \beta_2 \in [0,1]$, $\epsilon \geq 0$, max. iterations $T$.
    \State \textbf{Initialize} $m_{0} = 0$, $v_{0} = 0$
    \For{$t = 0, 1, \dots T$} \State
        \vspace*{-1.5em}
        \begin{align*}
            \theta_{t+1} &= \theta_{t} - \eta \cdot \left(
                \mathbf{diag}(v_{t}) + \epsilon I
            \right)^{-1/2} m_{t}; \\
            m_{t+1}        &= \frac{1}{1 - \beta_1^{t+1}} \left(\beta_{1} m_{t} + (1 - \beta_1) \grad \loss(\theta_{t+1})\right) \\
            v_{t+1}        &= \frac{1}{1 - \beta_2^{t+1}} \left(
                \beta_{2} v_{t} + (1 - \beta_2) (\grad \loss(\theta_{t+1}) \odot \grad \loss(\theta_{t+1}))
            \right) 
        \end{align*}
        \vspace*{-0.5em}
    \EndFor
    \State \Return $\theta_T$.
\end{algorithmic}
\label{alg:adam}
\end{algorithm}

\begin{algorithm}[h]
    \caption{$\SignGD$}
    \begin{algorithmic}[1]
        \State \textbf{Input}: First-order oracle $(\loss(\cdot), \grad \loss(\cdot))$, initialization $\theta_0 \in \R^p$, stepsize $\eta > 0$, $\epsilon \geq 0$, max. iterations $T$.
        \For{$t = 0, 1, \dots T$}
            \State $\theta_{t+1} = \theta_t - \eta \cdot \left(
                \mathbf{diag}\left(
                    \grad \loss(\theta_t) \odot \grad \loss(\theta_t)
                \right)
                + \epsilon I
            \right)^{-1/2} \grad \loss(\theta_t)$
        \EndFor
        \State \Return $\theta_T$.
    \end{algorithmic}
    \label{alg:signgd}
\end{algorithm}

\begin{algorithm}[h]
\caption{\texttt{GradientDescent}}
\begin{algorithmic}[1]
    \State \textbf{Input}: First-order oracle $(\loss(\cdot), \grad \loss(\cdot))$, initialization $\theta_0 \in \R^p$,  stepsize $\eta > 0$, max. iterations $T$.
    \For{$t = 0, 1, \dots T$}
        \State $\theta_{t+1} = \theta_t - \eta \cdot \grad \loss(\theta_t)$
    \EndFor
    \State \Return $\theta_T$.
\end{algorithmic}
\label{alg:SGD}
\end{algorithm}

\begin{algorithm}[H]
    \caption{EGOP-Reparameterization}\label{alg:meta-algorithm-block}
    \begin{algorithmic}[1]
        \State \textbf{Input}: First-order oracle $(\loss(\cdot), \grad \loss(\cdot))$, number of samples $M$, distribution $\rho$.
        \State Form empirical EGOP matrix
        \[
            \widehat{P} = \frac{1}{M}
            \sum_{i = 1}^{M} \grad \loss(\theta_{i}) \grad \loss(\theta_{i})^{\T}, \;\;
            \text{where} \;\;
            \set{\theta_{i}}_{i=1}^{M} \;{\sim}_{\mathrm{i.i.d.}} \;
            \rho.
        \]
        \State Optimize reparameterized function with adaptive algorithm of choice:
        \[
            {\widetilde{\theta}}^{\ast} =
            \argmin_{\theta \in \mathbb{R}^d} \loss(V\theta), \quad
            \text{where} \;\;
            V = \mathtt{eigenvectors}(\widehat{P})
        \]
        \State \Return $\theta^{\ast} \defeq V \widetilde{\theta}^{\ast}$.
    \end{algorithmic}
\end{algorithm}



%% file: src/deferred_proofs.tex
\section{Proofs}\label{sec:deferred-proofs}

\subsection{Proofs from Section~\ref{ssec:learning-with-rotated-data}}

To prove \Cref{lemma:update-for-any-deg-rotation}, we first characterize the (sub)gradients of the rotated expected loss in \eqref{eq:population_loss}. This result is a straightforward consequence of Proposition 2 in \cite{vasudeva2025rich} (up to normalization factors), but we present the derivation here for completeness.

\begin{lemma}\label{lemma:grads-under-rotation}
    Given $\omega, \mu, \sigma$ satisfying Assumption~\ref{assumption:realizability}, define the vectors
    \begin{subequations}
    \begin{align}
        \mu_{+} &\defeq \frac{\mu}{2} \begin{pmatrix}
            \omega - 1/\omega \\
            2
        \end{pmatrix} = \mathbb{E}_{(x,y) \sim \mathcal{D}}[x \mid y = 1, \epsilon = +1], \\
        \mu_{\pm} &\defeq \frac{\mu}{2} \begin{pmatrix}
            \omega -1/\omega \\
            -2
        \end{pmatrix} =
        \mathbb{E}_{(x, y) \sim \mathcal{D}}[x \mid y = 1, \epsilon = -1], \\
        \mu_{-} &\defeq \frac{\mu}{2} \begin{pmatrix}
            -(\omega + 1/\omega) \\
            0
        \end{pmatrix} =
        \mathbb{E}_{(x, y) \sim \mathcal{D}}[x \mid y = -1].
    \end{align}   
    \end{subequations}
    Let $\phi(\cdot)$ and $\Phi(\cdot)$ denote the PDF and CDF of the standard Gaussian distribution, respectively. For a vector of weights $w \in \R^2$, we define
    \begin{align}
    \left.
    \begin{aligned}
       p_{\alpha} &\defeq \Phi\left(\frac{1}{\sigma}\cdot \langle \Ugamma {\mu}_{\alpha}, \bar{w}\rangle\right), \\
       \Gamma_{\alpha} &\defeq \frac{\phi\left(\dfrac{1}{\sigma }\cdot \langle \Ugamma {\mu}_{\alpha}, \bar{w}\rangle\right)}{\Phi\left(\dfrac{1}{\sigma}\cdot \langle \Ugamma {\mu}_{\alpha}, \bar{w}\rangle\right)},
    \end{aligned} \qquad \right\} \;\;
    \text{for} \;\; \alpha \in \set{+, \pm, -}. \label{eq:def-p_pmzero-gamma-pmzero}
    \end{align}
    Then the subgradient of $\loss^\rot(W; f)$ with respect to $w_k \equiv W_{k, :} \in \R^2$ is
    \[
        \grad_{w_k} \loss^\rot(W; f) = -\frac{\sigma a_k}{4} \left(\sigma^{-1} \Ugamma (p_+ \mu_+ + p_{\pm} \mu_{\pm} -2 p_{-} \mu_{-}) + \bar{w}_k (p_+\Gamma_+ + p_{\pm} \Gamma_{\pm} -2p_{-}\Gamma_{-})
        \right).
    \]
    Here $a_k$ denotes the $k$th entry of outer layer weights $a\in \R^m$ and $\bar{w} \defeq w / \norm{w}$.
\end{lemma}

Lemma~\ref{lemma:grads-under-rotation} shows that the subgradient with respect to the $k^{\text{th}}$ row of $W$, $w_{k}$, only depends on $w_k$ via its \textit{direction}
$\nicefrac{w_k}{\norm{w_k}}$.

\begin{proof}[Proof of Lemma~\ref{lemma:grads-under-rotation}.]
    Given $\gamma > 0$, for clarity we denote samples from $\mathcal{D}^\rot$ as $(x^{\rot}, y^{\rot})$ to distinguish them from samples from $\mathcal{D}$.
    We will use the notation $\bar{w} \defeq w / \norm{w}$ to indicate the unit
    vector with the same direction as $w$.
    
    We begin by rewriting \Cref{eq:population_loss}:
    \begin{align*}
        \loss^\rot(W; f) &= \mathbb{E}_{\mathcal{D}^\rot}\left[-y^\rot \cdot a^\T \ReLU(W x^\rot)\right]\\
        &=\mathbb{E}_{\mathcal{D}^\rot} \left[-y^\rot \sum_{k=1}^m a_k \ReLU(\langle w_k, x^\rot\rangle ) \right].
    \end{align*}
    We will use $w_k$ to denote the weights of the $k$th hidden neuron, i.e., $w_k \defeq W_{k,:}$. For this row,
    \begin{align*}
        \nabla_{w_k} \loss^\rot(W; f) &= \mathbb{E}_{\mathcal{D}^\rot} \left[-y^\rot a_k \nabla_{w_k} \ReLU(\langle w_k, x^\rot\rangle ) \right]\\
        &= -a_k \mathbb{E}_{\mathcal{D}^\rot} \left[y^\rot x^\rot \indic\{\langle w_k, x^\rot\rangle \geq 0 \} \right] \\
        &= -a_k \mathbb{E}_{\mathcal{D}^\rot}\left[
            y^\rot x^\rot \indic\set{\ip{\bar{w}_k, x^\rot} \geq 0}
        \right],
    \end{align*}
    where the last equality follows from the fact that the indicator is invariant to scaling by a positive number.
    
    Recall that by the definition of $\mathcal{D}^\rot$ \cref{eq:def-rotated-dist},
    $
        \mathcal{D}^\rot\left((x, y)\right) = \mathcal{D}\left((\Ugamma ^\T x, y)\right)
    $.
    Thus for any function $g:\R^d\times \R \rightarrow \R$ whose expectation under $\mathcal{D}$ is finite, it holds that
    \[
        \mathbb{E}_{ \mathcal{D}^\rot}\left[ g\left((x^\rot, y^\rot)\right)\right] = \mathbb{E}_{\mathcal{D}}\left[ g\big((U(\gamma) x, y)\big)\right].
    \]
    In what follows, we will write $U \equiv \Ugamma$ for brevity.
    
    Consequently, we can rewrite our expectation as an expectation over $\mathcal{D}$:
    \begin{align*}
        \grad_{w_k} \loss^\rot(W) &=-a_k \mathbb{E}_{\mathcal{D}^\rot}\left[
            y^\rot x^\rot \indic\set{\ip{\bar{w}_k, x^\rot} \geq 0}
        \right]\\
        &=-a_k \mathbb{E}_{\mathcal{D}}\left[
            y Ux \indic\set{\ip{\bar{w}_k, Ux} \geq 0}
        \right]\\
        &=
        -a_k U \mathbb{E}_{\mathcal{D}}\left[
            y \cdot x \cdot \indic{\set{
                \ip{U^\T \bar{w}, x} \geq 0
            }}
        \right] \\
        &=
        -a_k U \mathbb{E}_{\mathcal{D}}\left[
            y \cdot x \cdot \indic{\set{
                \ip{\bar{w}_{\gamma}, x} \geq 0
            }}
        \right], \quad
        \text{where} \;\; \bar{w}_{\gamma} \defeq U^\T \bar{w}.
    \end{align*}
    In particular, note that since $U \in O(2)$, $U^\T \bar{w}$ is
    itself a unit vector. By the tower rule,
    \begin{equation}
        \mathbb{E}_{(x,y)}\left[
            y \cdot x \cdot \indic{\set{
                \ip{\bar{w}_{\gamma}, x} \geq 0
            }}
        \right] =
        \mathbb{E}_{y,\epsilon}\left[
            y \cdot \mathbb{E}_{x}\left[
                x \cdot \indic{\set{
                    \ip{\bar{w}_{\gamma}, x} \geq 0
                }}
            \mid y, \epsilon \right]
        \right]
        \label{eq:tower}
    \end{equation}
    We now simplify the inner expectation in~\eqref{eq:tower}. Since all required densities exist and are bounded away from zero, we can rewrite it as
    \begin{align}
        \mathbb{E}_{x}\left[
            x \cdot \indic{
            \set{\ip{\bar{w}_{\gamma}, x} \geq 0}
            } \mid y, \epsilon
        \right] &=
        \mathbb{E}_{x}\left[
            x \mid 
            \ip{\bar{w}_{\gamma}, x} \geq 0,
            y, \epsilon
        \right] \cdot
        \mathbb{P}\left(
        \ip{\bar{w}_{\gamma}, x} \geq 0 \mid
        y, \epsilon
        \right)
        \label{eq:expectation-after-tower}
    \end{align}
    For fixed $\bar{w}_{\gamma}$, we decompose $x$ into two components, parallel and orthogonal to $\bar{w}_{\gamma}$, respectively:
    \begin{equation}
        x_{\myparallel} \defeq \ip{x, \bar{w}_{\gamma}} \bar{w}_{\gamma}, \quad \text{and} \quad
        x_{\perp} \defeq x - x_{\myparallel} \implies
        \ip{\bar{w}_{\gamma}, x} =
        \ip{\bar{w}_{\gamma}, x_{\myparallel}}.
        \label{eq:x-orthodecomp}
    \end{equation}
    Note that $x_{\perp}$ is independent of
    $x_{\myparallel}$. Hiding the conditioning
    on $y$, $\epsilon$ for brevity, we have
    \begin{align}
        \mathbb{E}_{x}\left[
            x \mid \ip{\bar{w}_{\gamma}, x} \geq 0
        \right] &=
        \mathbb{E}_{x}\left[
            x_{\myparallel} \mid
            \ip{\bar{w}_{\gamma}, x_{\myparallel}} \geq 0
        \right] +
        \mathbb{E}_{x}\left[
            x_{\perp} \mid
            \ip{\bar{w}_{\gamma}, x_{\myparallel}} \geq 0
        \right] \notag \\
        &=
        \mathbb{E}_{x}\left[
            x_{\myparallel} \mid
            \ip{\bar{w}_{\gamma}, x_{\myparallel}} \geq 0
        \right] +
        \mathbb{E}_{x}\left[x_{\perp}\right] \notag \\
        &=
        \mathbb{E}_{x}\left[
            x_{\myparallel} \mid
            \ip{\bar{w}_{\gamma}, x_{\myparallel}} \geq 0
        \right] +
        \mathbb{E}_{x}[x] - \mathbb{E}_{x}\left[
            x_{\myparallel}
        \right] \notag \\
        &=
        \bar{w}_{\gamma} \cdot \mathbb{E}_{x}\left[
            \ip{\bar{w}_{\gamma}, x_{\myparallel}} \mid
            \ip{\bar{w}_{\gamma}, x_{\myparallel}} \geq 0
        \right] +
        \mathbb{E}_{x}[x] - \mathbb{E}_{x}\left[
            x_{\myparallel}
        \right]
        \notag \\
        &=
        \bar{w}_{\gamma} \cdot \left( \mathbb{E}_{x}\left[
            \ip{\bar{w}_{\gamma}, x_{\myparallel}} \mid
            \ip{\bar{w}_{\gamma}, x_{\myparallel}} \geq 0
        \right] -
        \mathbb{E}_{x}[\ip{\bar{w}_{\gamma}, x_{\myparallel}}]
        \right) +
        \mathbb{E}_{x}[x].
        \label{eq:expectation-simplify-1}
    \end{align}
    Recall the following fact about truncated Gaussian distributions: if $X \sim \mathcal{N}(\mu, \sigma^2)$,
    \begin{equation}
        \label{eq:truncated-gaussian}
        \mathbb{E}[X \mid \alpha \leq X \leq \beta] =
        \mu - \sigma \cdot \frac{
            \phi\big(\frac{\beta - \mu}{\sigma}\big) - \phi\big(\frac{\alpha - \mu}{\sigma}\big)
            }{
                \Phi\big(\frac{\beta - \mu}{\sigma}\big) -
                \Phi\big(\frac{\alpha - \mu}{\sigma}\big)
            }
    \end{equation}
    Applying~\eqref{eq:truncated-gaussian} with $\alpha = 0$
    and $\beta \to \infty$ to the first expectation in~\eqref{eq:expectation-simplify-1}, we obtain
    \begin{align*}
        \mathbb{E}_{x}\left[
            \ip{\bar{w}_{\gamma}, x_{\myparallel}} \mid
            \ip{\bar{w}_{\gamma}, x_{\myparallel}} \geq 0
        \right] &=
        \mathbb{E}[\ip{\bar{w}_{\gamma}, x_{\myparallel}}]
        + \sigma \cdot \frac{\phi\left(-\frac{\mathbb{E}[\ip{\bar{w}_{\gamma}, x_{\myparallel}}]}{\sigma}
        \right)}{1 - \Phi\left(
            -\frac{\mathbb{E}[\ip{\bar{w}_{\gamma}, x_{\myparallel}}]}{\sigma}
        \right)} \\
        &=
        \mathbb{E}[\ip{\bar{w}_{\gamma}, x_{\myparallel}}]
        + \sigma \cdot \frac{\phi\left(\frac{\mathbb{E}[\ip{\bar{w}_{\gamma}, x_{\myparallel}}]}{\sigma}
        \right)}{\Phi\left(
            \frac{\mathbb{E}[\ip{\bar{w}_{\gamma}, x_{\myparallel}}]}{\sigma}
        \right)},
    \end{align*}
    using the identities $\phi(x) = \phi(-x)$ and
    $1 - \Phi(-x) = \Phi(x)$ and the fact that
    the transformed random variable has equal variance. Plugging back into~\eqref{eq:expectation-simplify-1} yields
    \begin{equation}
        \mathbb{E}_{x}[x \mid \ip{\bar{w}_{\gamma}, x} \geq 0] =
        \sigma \bar{w}_{\gamma} \cdot
        \frac{
            \phi\left(-\frac{\mathbb{E}[\ip{\bar{w}_{\gamma}, x_{\myparallel}}]}{\sigma}\right)
        }{
            1 - \Phi\left(
                -\frac{\mathbb{E}[\ip{\bar{w}_{\gamma}, x_{\myparallel}}]}{\sigma}
            \right)
        } + \mathbb{E}_{x}[x].
        \label{eq:expectation-simplify-1+}
    \end{equation}
    The remaining expectations in\eqref{eq:expectation-simplify-1+} are equal to
    \begin{align*}
        \mathbb{E}[\ip{\bar{w}_{\gamma}, x_{\myparallel}} \mid y, \epsilon] &=
        \mathbb{E}[\ip{\bar{w}_{\gamma}, x} \mid y, \epsilon]
        =
        \frac{\mu}{2} \cdot \ip*{\bar{w}_{\gamma}, \begin{pmatrix}
            \left(y \omega - \frac{1}{\omega}\right) \\
            \epsilon (y + 1)
        \end{pmatrix}} \\
        \mathbb{E}[x \mid y, \epsilon] &=
        \frac{\mu}{2} \begin{pmatrix}
            y \omega - \frac{1}{\omega} \\
            \epsilon (y + 1)
        \end{pmatrix}.
    \end{align*}
    Before we proceed, we define the following events and quantities for brevity:
    \begin{align}
        \mathcal{E}_{+} &\defeq \set{
        y = \epsilon = 1
        }, \;\;
        \mathcal{E}_{\pm} \defeq \set{
        y = 1, \epsilon = -1
        }, \;\;
        \mathcal{E}_{-} \defeq \set{
        y = -1
        }
        \label{eq:e-events-omg} \\
        \mu_{+} &\defeq \frac{\mu}{2} \begin{pmatrix*}[c]
            \omega - \nicefrac{1}{\omega} \\
            2
        \end{pmatrix*}, \;\;
        \mu_{\pm} \defeq \frac{\mu}{2} \begin{pmatrix*}[c]
            \omega - \nicefrac{1}{\omega} \\
            -2
        \end{pmatrix*}, \;\;
        \mu_{-} \defeq \frac{\mu}{2} \begin{pmatrix*}[c]
            -\left(\omega + \nicefrac{1}{\omega}\right) \\
            0
        \end{pmatrix*}
        \label{eq:mus-omg} \\
        \Gamma_{+} &\defeq \frac{
            \phi\left(
                \frac{\ip{\bar{w}_{\gamma}, \mu_{+}}}{\sigma}
            \right)
        }{
            \Phi\left(
                \frac{\ip{\bar{w}_{\gamma}, \mu_{+}}}{\sigma}
            \right)
        }, \;\;
        \Gamma_{\pm} \defeq \frac{
            \phi\left(
                \frac{\ip{\bar{w}_{\gamma}, \mu_{\pm}}}{\sigma}
            \right)
        }{
            \Phi\left(
                \frac{\ip{\bar{w}_{\gamma}, \mu_{\pm}}}{\sigma}
            \right)
        }, \;\;
        \Gamma_{-} \defeq \frac{
            \phi\left(
                \frac{\ip{\bar{w}_{\gamma}, \mu_{-}}}{\sigma}
            \right)
        }{
            \Phi\left(
                \frac{\ip{\bar{w}_{\gamma}, \mu_{-}}}{\sigma}
            \right)
        }.
        \label{eq:gammas-omg}
    \end{align}
    Since $y$ and $\epsilon$ are iid Rademacher, we have
    \begin{equation}
        \prob{\mathcal{E}_{+}} = \prob{\mathcal{E}_{\pm}} = \frac{1}{4},
        \;\;
        \prob{\mathcal{E}_{-}} = \frac{1}{2}.
        \label{eq:e-events-prob}
    \end{equation}
    Depending on the values of $y$ and $\epsilon$, the above yields
    \begin{subequations}
    \begin{align}
        \mathbb{E}[\ip{\bar{w}_{\gamma}, x_{\myparallel}} \mid \mathcal{E}_{+}] &=
        \mu \cdot
        \left[
        \left(\omega - \frac{1}{\omega}\right) (\bar{w}_{\gamma})_{1} +
        (\bar{w}_{\gamma})_{2}
        \right] = \ip{\bar{w}_{\gamma}, \mu_{+}} \label{eq:expectation-simplify-2a} \\
        \mathbb{E}[\ip{\bar{w}_{\gamma}, x_{\myparallel}} \mid \mathcal{E}_{\pm}] &=
        \mu \cdot
        \left[
        \left(\omega - \frac{1}{\omega}\right) (\bar{w}_{\gamma})_{1} -
        (\bar{w}_{\gamma})_{2}
        \right] =
        \ip{\bar{w}_{\gamma}, \mu_{\pm}}
        \label{eq:expectation-simplify-2b} \\
        \mathbb{E}[\ip{\bar{w}_{\gamma}, x_{\myparallel}} \mid \mathcal{E}_{-}] &=
        -\mu \cdot
        \left(\omega + \frac{1}{\omega}\right)
        (\bar{w}_{\gamma})_{1} =
        \ip{\bar{w}_{\gamma}, \mu_{-}}
        \label{eq:expectation-simplify-2c}
    \end{align}
    \end{subequations}
    We can also calculate the corresponding truncation probabilities conditional on $y$ and $\epsilon$:
    \begin{subequations}
        \begin{align}
            \prob{
            \ip{\bar{w}_{\gamma}, x_{\myparallel}}
            \geq 0 \mid \mathcal{E}_{+}
            } &=
            \Phi\left(
            \frac{\mu}{\sigma} \left(
            (\bar{w}_{\gamma})_{1} \left(\omega - \frac{1}{\omega}\right) +
            (\bar{w}_{\gamma})_{2}
            \right)
            \right) \notag \\ 
            &=
            \Phi\left(
            \frac{\ip{\bar{w}_{\gamma}, \mu_{+}}}{\sigma}
            \right)
            \label{eq:probability-simplify-1a} \\
            \prob{
            \ip{\bar{w}_{\gamma}, x_{\myparallel}}
            \geq 0 \mid \mathcal{E}_{\pm}}
            &=
            \Phi\left(
            \frac{\mu}{\sigma} \left(
            (\bar{w}_{\gamma})_{1} \left(\omega - \frac{1}{\omega}\right) -
            (\bar{w}_{\gamma})_{2}
            \right)
            \right) \notag \\ 
            &=
            \Phi\left(\frac{\ip{\bar{w}_{\gamma}, \mu_{\pm}}}{\sigma}\right)
            \label{eq:probability-simplify-1b} \\
            \prob{
            \ip{\bar{w}_{\gamma}, x_{\myparallel}}
            \geq 0 \mid \mathcal{E}_{-}} &=
            \Phi\left(
            \frac{\mu}{\sigma} \left(
            -(\bar{w}_{\gamma})_{1} \left(\omega + \frac{1}{\omega}\right)
            \right)
            \right) \notag \\
            &=
            \Phi\left(
            \frac{\ip{\bar{w}_{\gamma}, \mu_{-}}}{\sigma}
            \right)
            \label{eq:probability-simplify-1c}
        \end{align}
    \end{subequations}
    Finally, we calculate the overall expectation
    in~\eqref{eq:tower} using the law of total expectation:
    \begin{align*}
        & \mathbb{E}_{\mathcal{D}}[
        y \cdot x \indic{\set{\ip{\bar{w}_{\gamma}, x} \geq 0}}
        ] \\
        &= \frac{1}{4} \cdot \left(
            \mathbb{E}[x \indic{\ip{\bar{w}_{\gamma}, x} \geq 0} \mid \mathcal{E}_{+}]
            + \mathbb{E}[x \indic{ \ip{\bar{w}_{\gamma}, x} \geq 0}
            \mid \mathcal{E}_{\pm}]
            - 2 \mathbb{E}[x \indic{ \ip{\bar{w}_{\gamma}, x} \geq 0}
            \mid
            \mathcal{E}_{-}]
        \right) \\
        &= \frac{\sigma}{4} \left(
            \Big(\bar{w}_{\gamma} \Gamma_{+}
             + \frac{\mu_{+}}{\sigma}\Big) \Phi(\mu_{+})
          + \Big(\bar{w}_{\gamma} \Gamma_{\pm}
            + \frac{\mu_{\pm}}{\sigma}\Big)
            \Phi(\mu_{\pm})
        - 2 \Big(\bar{w}_{\gamma} \Gamma_{-}
            + \frac{\mu_{-}}{\sigma}\Big)
            \Phi(\mu_{-})
        \right),
    \end{align*}
    using~\cref{eq:expectation-simplify-2a,eq:expectation-simplify-2b,eq:expectation-simplify-2c,eq:probability-simplify-1a,eq:probability-simplify-1b,eq:probability-simplify-1c} throughout. Writing
    \begin{equation}
        p_{\alpha} = \Phi\left(
        \frac{\ip{\bar{w}_{\gamma}, \mu_{\alpha}}}{\sigma}
        \right), \;\;
        \text{for $\alpha \in \set{+, -, \pm}$},
        \label{eq:cdfs-shorthand}
    \end{equation}
    and recalling that $\bar{w}_{\gamma} = U^\T \bar{w}$,
    we can simplify
    \begin{align*}
        & \grad_{w_{k}} \loss^{\rot}(W; f) \\
        &= -a_{k} \Ugamma \mathbb{E}_{\mathcal{D}}[
        y \cdot x \indic{\set{\ip{\bar{w}_{\gamma}, x} \geq 0}}
        ] \\
        &= -\frac{\sigma a_{k}}{4} \left(
            U \bar{w}_{\gamma} \left(
                p_{+} \Gamma_{+} +
                p_{\pm} \Gamma_{\pm} -
                2 p_{-} \Gamma_{-}
            \right) +
            \sigma^{-1} U
            \left(
            \mu_{+} p_{+} + \mu_{\pm} p_{\pm}
            - 2 \mu_{-} p_{-}
            \right)
        \right)\\
        &= -\frac{\sigma a_{k}}{4} \left(
            \bar{w}_{k} \left(
                p_{+} \Gamma_{+} +
                p_{\pm} \Gamma_{\pm} -
                2 p_{-} \Gamma_{-}
            \right) +
            \sigma^{-1} U
            \left(
            \mu_{+} p_{+} + \mu_{\pm} p_{\pm}
            - 2 \mu_{-} p_{-}
            \right)
        \right).
    \end{align*}
    Finally, recognizing $\ip{\bar{w}_{\gamma}, \mu_{+}} = \ip{U^\T \bar{w}, \mu_{+}}=\ip{\bar{w}, U \mu_{+}}$ (and similarly for $\mu_{\pm}, \mu_{-}$) yields the claim.
\end{proof}

In order to prove Lemma~\ref{lemma:update-for-any-deg-rotation}, we will make use of the following result:
\begin{lemma}\label{lemma:cmu-defn}
    Consider any $\omega > 1$ and $S = \nicefrac{\mu}{\sigma} \in (0, \infty)$. For $p_{+}$, $p_{\pm}$ and $p_{-}$ as defined in \cref{eq:def-p_pmzero-gamma-pmzero}, any $\gamma \in \R$, and any $w \in \R^2$, it holds that
    \begin{equation} \label{eq:def-cmu}
        p_+(w) + p_{\pm}(w) + 2p_{-}(w) \geq \cmu,
        \quad \text{where} \quad
        \cmu \defeq
        \Phi\left(
            S \cdot \min\set*{
                1,
                \frac{\omega - \nicefrac{1}{\omega}}{2}
            }
        \right)
    \end{equation}
    We note that for any $\omega > 1$ and any $\mu, \sigma > 0$, $\cmu \in [0.5, 1]$.
\end{lemma}
\begin{proof}[Proof of Lemma~\ref{lemma:cmu-defn}]
    We note that $p_{+}$, $p_{\pm}$ and $p_{-}$, defined in \cref{lemma:update-for-any-deg-rotation}, are invariant under rescalings of $\norm{w}$. Thus without loss of generality, we prove the result for $w$ such that $\norm{w} = 1$.
    
    Recall that the rotation matrix $U \equiv \Ugamma$ is
    \[
        \Ugamma = \begin{bmatrix*}[r]
            \cos(\gamma) & -\sin(\gamma)\\
            \sin(\gamma) & \cos(\gamma)
        \end{bmatrix*} = \begin{bmatrix}
            u_{1} & u_{2}
        \end{bmatrix}, \;\;
        \text{where} \quad
        u_{1} = \begin{pmatrix*}[r]
            \cos(\gamma) \\ \sin(\gamma)
        \end{pmatrix*} \text{ and }
        u_{2} = \begin{pmatrix*}[r]
            -\sin(\gamma) \\ \cos(\gamma)
        \end{pmatrix*}.
    \]
    By the preceding display and the definitions of $\mu_{+}$,
    $\mu_{\pm}$ and $\mu_{-}$ in~\cref{lemma:update-for-any-deg-rotation}, we have that
    \begin{subequations}
        \begin{align}
            U \mu_{+} &=
            \mu \left( \frac{\omega - \nicefrac{1}{\omega}}{2} u_{1}
            + u_{2}\right)
            \label{eq:umu-a}
            \\
            U \mu_{\pm} &=
            \mu \left( \frac{\omega - \nicefrac{1}{\omega}}{2} u_{1}
            - u_{2} \right)
            \label{eq:umu-b}
            \\
            U \mu_{-} &=
            -\mu \left(\frac{\omega + \nicefrac{1}{\omega}}{2} u_{1}\right)
            \label{eq:umu-c}
        \end{align}
    \end{subequations}
    Since $u_{1}$ and $u_{2}$ are an orthonormal basis of $\mathbb{R}^2$, we can write any $\bar{w} \in \mathbb{S}^1$ as
    \[
        \bar{w} = \alpha u_{1} + \beta u_{2}, \quad
        \text{where} \;\; \alpha^2 + \beta^2 = 1.
    \]
    Using these values and the expressions derived above, we have
    \begin{align*}
        p_{+} &=
        \Phi\left(\frac{1}{\sigma} \ip{U \mu_{+}, \bar{w}}\right) =
        \Phi\left(
        \frac{\mu}{\sigma} \left(
            \frac{\alpha(\omega - \nicefrac{1}{\omega})}{2}
            + \beta
        \right)
        \right), \\
        p_{\pm} &=
        \Phi\left(\frac{1}{\sigma} \ip{U \mu_{\pm}, \bar{w}}\right) =
        \Phi\left(
        \frac{\mu}{\sigma} \left(
            \frac{\alpha(\omega - \nicefrac{1}{\omega})}{2}
            - \beta
        \right)
        \right), \\
        p_{-} &=
        \Phi\left(\frac{1}{\sigma} \ip{U \mu_{-}, \bar{w}}\right) =
        \Phi\left(
        -\frac{\mu}{\sigma} \left(
            \frac{\alpha(\omega + \nicefrac{1}{\omega})}{2}
        \right)
        \right).
    \end{align*}
    We now proceed on a case-by-case basis.
    \paragraph{The case where $\alpha \leq 0$.}
    Under this condition, since $\frac{\mu}{\sigma}$ and $\omega$ are nonnegative, we have
    \[
        -\frac{\mu}{\sigma} \frac{\alpha(\omega + \nicefrac{1}{\omega})}{2} \geq 0 \implies
        p_{-} = \Phi\left(
        -\frac{\mu}{\sigma} \left(
        \frac{\alpha(\omega + \nicefrac{1}{\omega})}{2}
        \right)
        \right) \geq \frac{1}{2},
    \]
    thus the entire sum is bounded from below by $1$.
    \paragraph{The case where $\alpha > 0$.} By monotonicity of the Gaussian CDF, we have that
    \begin{align*}
        \max\set{p_{+}, p_{\pm}} &=
        \max\set*{
            \Phi\left(
                \frac{\mu}{\sigma} \left(\frac{\alpha(\omega - \nicefrac{1}{\omega})}{2} + \beta
                \right)
            \right), 
            \Phi\left(
                \frac{\mu}{\sigma} \left(\frac{\alpha(\omega - \nicefrac{1}{\omega})}{2} - \beta
                \right)
            \right)
        } \\
        &=
        \Phi\left(
            \frac{\mu}{\sigma} \left(\frac{\alpha(\omega - \nicefrac{1}{\omega})}{2} + \abs{\beta}
            \right)
        \right) \\
        &=
        \Phi\left(
            \frac{\mu}{\sigma} \left(\frac{\alpha(\omega - \nicefrac{1}{\omega})}{2} + \sqrt{1 - \alpha^2}
            \right)
        \right),
    \end{align*}
    where the last equality follows from the fact that
    $\alpha^2 + \beta^2 = 1$ implies $\beta = \pm \sqrt{1 - \alpha^2}$. Analytically minimizing the inner expression
    over $\alpha \in (0, 1]$, we find
    \begin{align*}
        \min_{\alpha \in [0, 1]} \set[\Big]{
        \frac{\alpha}{2} \left(\omega - \nicefrac{1}{\omega}\right) + \sqrt{1 - \alpha^2}} =
        \min\set*{
        1, \frac{\omega - \nicefrac{1}{\omega}}{2}
        }.
    \end{align*}
    We conclude that $p_{+} + p_{\pm} \geq \max\set{p_{+}, p_{\pm}} \geq \Phi\left(\frac{\mu}{\sigma}
    \min\set*{1, \frac{\omega - \nicefrac{1}{\omega}}{2}}
    \right)$.

    \paragraph{Putting it all together.}
    Collecting both cases, we obtain the inequality
    \begin{equation}
        \min_{w \in \mathbb{S}^{1}} \set*{
            p_{+}(w) + p_{\pm}(w) + 2p_{-}(w)
        } \geq
        \Phi\left(
        \frac{\mu}{\sigma} \min\set*{
            1, \frac{\omega - \nicefrac{1}{\omega}}{2}
        }
        \right).
    \end{equation}
\end{proof}

\begin{proof}[Proof of Lemma~\ref{lemma:update-for-any-deg-rotation}]
    We will denote $w \defeq w_k$ and $U\defeq \Ugamma$. Recall that $U$ equals
    \[
        \Ugamma = \begin{bmatrix*}[r]
            \cos(\gamma) & -\sin(\gamma)\\
            \sin(\gamma) & \cos(\gamma)
        \end{bmatrix*} = \begin{bmatrix}
            u_{1} & u_{2}
        \end{bmatrix}, \;\;
        \text{where} \quad
        u_{1} = \begin{pmatrix*}[r]
            \cos(\gamma) \\ \sin(\gamma)
        \end{pmatrix*} \text{ and }
        u_{2} = \begin{pmatrix*}[r]
            -\sin(\gamma) \\ \cos(\gamma)
        \end{pmatrix*}.
    \]
    From these definitions,~\cref{eq:umu-a,eq:umu-b,eq:umu-c}
    and~\cref{lemma:grads-under-rotation}, it follows that
    \begin{align*}
        & \grad_{w_k} \loss^{\rot}(W; f) \\
        &= -\frac{\sigma a_{k}}{4} \left(
            \sigma^{-1} U
            \left(
            \mu_{+} p_{+} + \mu_{\pm} p_{\pm}
            - 2 \mu_{-} p_{-}
            \right)
            +
            \left(
                p_{+} \Gamma_{+} +
                p_{\pm} \Gamma_{\pm} -
                2 p_{-} \Gamma_{-}
            \right)\cdot \bar{w}
        \right)\\
        &= \begin{aligned}[t]
        & -\frac{\sigma a_k}{4}
        \left(
            \frac{\mu}{\sigma} \left[
            \left(
                \omega \left(p_{+} + p_{\pm} + 2 p_{-}\right)
              - \frac{1}{\omega} \left(
                    p_{+} + p_{\pm} - 2p_{-}
                \right)
            \right) \frac{u_{1}}{2} +
            (p_{+} - p_{\pm}) u_{2}
            \right]
        \right) \\
        & -\frac{\sigma a_k}{4}
        \left(p_{+} \Gamma_{+} + p_{\pm} \Gamma_{\pm}
        - 2p_{-} \Gamma_{-}\right) \cdot \bar{w}
        \end{aligned}
    \end{align*}
%
    Denoting $\bar{w}=[\bar{w}_1, \bar{w}_2]^\T$,  the entries of the gradient are thus
    \begin{align}
        & \left[\grad_{w_k} \loss^{\rot}(W)\right]_{1}
        \label{eq:grad-entry-1} \\
        &= \begin{aligned}[t]
        & -\frac{\sigma a_k}{4} \left(
        \frac{\mu}{\sigma} \left[ \left(
        (p_{+} + p_{\pm} + 2p_{-})
        \omega - (p_{+} + p_{\pm} - 2p_{-}) \nicefrac{1}{\omega}
        \right) \cos(\gamma) -
        (p_{+} - p_{\pm}) \sin(\gamma)
        \right]
        \right) \\
        & -\frac{\sigma a_k}{4} \left(
        p_{+} \Gamma_{+} + p_{\pm} \Gamma_{\pm} - 2p_{-} \Gamma_{-}
        \right) \bar{w}_1
        \end{aligned} \notag \\[1em]
        & \left[\grad_{w_k} \loss^{\rot}(W)\right]_{2}
        \label{eq:grad-entry-2} \\
        &= \begin{aligned}[t]
        & -\frac{\sigma a_k}{4} \left(
        \frac{\mu}{\sigma} \left[ \left(
        (p_{+} + p_{\pm} + 2p_{-})
        \omega - (p_{+} + p_{\pm} - 2p_{-}) \nicefrac{1}{\omega}
        \right) \sin(\gamma) +
        (p_{+} - p_{\pm}) \cos(\gamma)
        \right]
        \right) \\
        & -\frac{\sigma a_k}{4} \left(
        p_{+} \Gamma_{+} + p_{\pm} \Gamma_{\pm} - 2p_{-} \Gamma_{-}
        \right) \bar{w}_2
        \end{aligned} \notag
    \end{align}

    We proceed to bound the gradient entries based on the value of $a_{k} \in \set{\pm 1}$.

    \paragraph{Case 1: $a_{k} = 1$.}
    To upper bound~\eqref{eq:grad-entry-1}, we note that
    \begin{equation}
        \abs{p_{+} + p_{\pm} - 2p_{-}} \leq 2,
        \;\;
        \abs{p_{+} - p_{\pm}} \leq 1,
        \;\; \text{ and } \;\;
        \abs{p_{+} \Gamma_{+} + p_{\pm} \Gamma_{\pm}
        - 2 p_{-} \Gamma_{-}} \leq
        \frac{2}{\sqrt{2 \pi}},
        \label{eq:p-bounds}
    \end{equation}
    where the last inequality follows from bounding the standard Gaussian density, and from the fact that $\cos(\gamma)$ and $\sin(\gamma)$ are both
    nonnegative when $\gamma \in [0, \pi/4]$.
    At the same time, \cref{lemma:cmu-defn} shows that
    \(
        p_{+} + p_{\pm} + 2 p_{-} \geq \cmu.
    \)
    Combining all the above bounds, we arrive at
    \begin{align*}
        \left[ \grad_{w_{k}} \loss^{\rot}(W) \right]_{1}
        &\leq
        -\frac{\sigma}{4}
        \set*{\left[
        \frac{\mu}{\sigma}\left(\left(\cmu \omega
        - \frac{2}{\omega}\right)
        \cos(\gamma) - \sin(\gamma)
        \right)
        \right] - \frac{2}{\sqrt{2 \pi}}
        } \\
        \left[ \grad_{w_{k}} \loss^{\rot}(W) \right]_{2}
        &\leq
        -\frac{\sigma}{4}
        \set*{\left[
        \frac{\mu}{\sigma}\left(\left(\cmu \omega
        - \frac{2}{\omega}\right)
        \sin(\gamma) - \cos(\gamma)
        \right)
        \right] - \frac{2}{\sqrt{2 \pi}}
        }
    \end{align*}
    Thus, we have $\sign(\grad_{w_k} \loss^{\rot}(W)) = [-1, -1]^{\T}$ whenever
    \begin{align}
    \left\{
    \begin{aligned}
        \frac{\mu}{\sigma} \left(\cmu \omega - \frac{2}{\omega}
        \right) \cos(\gamma) - \sin(\gamma)
        - \frac{2}{\sqrt{2 \pi}} &> 0 \\
        \frac{\mu}{\sigma} \left(\cmu \omega - \frac{2}{\omega}
        \right) \sin(\gamma) - \cos(\gamma)
        - \frac{2}{\sqrt{2 \pi}} &> 0
    \end{aligned}
    \right.
    \label{eq:sign-condition-1}
    \end{align}

    We note that in order for the above inequality to have solutions, a necessary condition is
    \begin{equation}\label{eq:omega-bound}
        \cmu \omega - \frac{2}{\omega} >0 \quad \iff \quad \omega  > \sqrt{2/\cmu}.
    \end{equation}
    Moreover, for $\gamma \in (0, \pi/4]$, $\cos(\gamma) \geq \sin(\gamma) \geq 0$. Thus both inequalities in \eqref{eq:sign-condition-1} are satisfied whenever
    \begin{equation}\label{eq:simplified-sign-condition-1}
        \frac{\mu}{\sigma} \left(\cmu \omega - \frac{2}{\omega}
        \right) \sin(\gamma) - \cos(\gamma)
        - \frac{2}{\sqrt{2 \pi}} > 0.
    \end{equation}

    \paragraph{Case 2: $a_{k} = -1$.}
    Again making use of~\eqref{eq:p-bounds} and the expression for $\grad_{w_k} \loss(W)$, we obtain
    \begin{align*}
        \left[\grad_{w_k} \loss^{\rot}(W)\right]_{1} &\geq
        \frac{\sigma}{4} \left(
        \frac{\mu}{\sigma} \left[
        \left(\cmu \omega - \frac{2}{\omega}\right) \cos(\gamma)
        - \sin(\gamma)
        \right] - \frac{2}{\sqrt{2 \pi}}
        \right) \\
        \left[\grad_{w_k} \loss^{\rot}(W)\right]_{2} &\geq
        \frac{\sigma}{4} \left(
        \frac{\mu}{\sigma} \left[
        \left(\cmu \omega - \frac{2}{\omega}\right) \sin(\gamma)
        - \cos(\gamma)
        \right] - \frac{2}{\sqrt{2 \pi}}
        \right)
    \end{align*}
    By the preceding display, we have
    $\sign(\grad_{w_k} \loss^{\rot}(W)) = [1, 1]^{\T}$ precisely when both conditions in~\eqref{eq:sign-condition-1}
    are satisfied.

    Combining the results from the two cases $a_k = \pm 1$, we conclude that under~\cref{assumption:realizability,assumption:general-rotation-omega-mu-sigma-relationship}, it holds that
    \[
        \sign(\grad_{w_k} \loss^{\rot}(W))
        = - a_{k} \cdot [1, 1]^{\T}.
    \]
\end{proof}

\subsection{Proofs from Section~\ref{sec:EGOP-reparameterization}}\label{ssec:appendix-equivariance}

We begin by generalizing Theorem~\ref{thm:EGOP-invariant-decision-boundaries-2d-gamma} to a broad class of data distributions, loss functions, and rotations, including higher-dimensional settings. In \cref{sssec:general-equivariance}, we briefly define generalized notation and state Theorem~\ref{thm:EGOP-invariant-decision-boundaries}, which subsumes Theorem~\ref{thm:EGOP-invariant-decision-boundaries-2d-gamma} as a special case. In \cref{sssec:proof-generalized-equivariance}, we prove Theorem~\ref{thm:EGOP-invariant-decision-boundaries}. In \cref{sssec:eigenmatrix-computation-assumption}, we comment on issues of non-uniqueness of the eigenbasis in the case of repeated eigenvalues, and why these issues can be resolved without loss of generality.

\subsubsection{Equivariance for Generalized Distributions and Loss Functions}\label{sssec:general-equivariance}

In this section we prove that EGOP-reparameterization endows adaptive algorithms for a broad family of netowrk architectures, data distributions, and loss functions. Consider feature-label pairs $(x, y) \in \mathbb{R}^d \times \mathbb{R}$. Our results hold for any feature-label distribution $\mathcal{D}$ and any loss function $\ell:\R\times \R \rightarrow \R$ for which the population loss is well-defined, in the sense that
\begin{equation}\label{eq:general-population-loss}
        \mathcal{L}(\theta; f) \defeq \mathbb{E}_{\mathcal{D}}\left[\ell\big(y; f(\theta; x)\big)\right]
\end{equation}
is finite for any fixed $\theta \in \R^p$.

Given any population loss $\mathcal{L}(\cdot, \cdot)$ of the form in \cref{eq:general-population-loss}, we use $\mathcal{L}^{(U)}(\cdot, \cdot)$ to denote the loss when the data distribution is transformed by some $U\in O(d)$:
\begin{equation}\label{eq:rotated-general-population-loss}
    \mathcal{L}^{(U)}(\theta; f) \defeq \mathbb{E}_{\mathcal{D}^{(U)}}\left[\ell\big(y; f(\theta; x)\big)\right].
\end{equation}
In the above, $\mathcal{D}^{(U)}(\cdot, \cdot)$ is as-defined in \cref{eq:def-rotated-dist}. 

The EGOP-reparameterization method leverages the EGOP matrix of the expected loss,
$\mathrm{EGOP}_{\rho}(\loss^{(U)})$ (cf.~\cref{def:egop-matrix}). Writing $V_{U} \in O(p)$ for the eigenbasis of the
EGOP matrix for data distribution $\mathcal{D}^{(U)}$,
\begin{equation}\label{eq:EGOP-eigenbasis}
    V_{U} \defeq \mathtt{eigenvectors}(\mathrm{EGOP}_{\rho}(\mathcal{L}^{(U)}(\cdot\,; f))),
\end{equation}
we define the EGOP-reparameterized network $f_{U}(\cdot\,; \cdot)$ 
\begin{equation}\label{eq:def-EGOP-forward-map}
    f_{U}(\theta; x) \defeq f(V_{U} \theta; x), 
\end{equation}
and the corresponding loss $\mathcal{L}^{(U)}(\theta; f_{U})$, which is an instantiation of \cref{eq:rotated-general-population-loss} with EGOP-reparameterized forward map $f_{U}(\cdot; \cdot)$.

We consider the class of deterministic algorithms, defined in \cref{eq:deterministic-update}, and prove that EGOP-reparameterization endows such algorithms with equivariance to orthogonal transformations in data space.
\begin{theorem}\label{thm:EGOP-invariant-decision-boundaries}
    Consider any family of feed-forward networks of the form in Definition~\ref{def:multilayer-network} and any deterministic, first-order algorithm $\mathcal{A}$. Consider data distribution $\mathcal{D}$, population loss $\mathcal{L}(\cdot, \cdot)$, and EGOP sampling distribution $\rho$ satisfying \cref{assumption:rotational-invariance}. For any $U\in O(d)$, let $\theta^{(U)}_T$ denote the result of optimizing the EGOP-reparameterized objective $\mathcal{L}^{(U)}(\cdot; f_{U})$ as defined in \cref{eq:rotated-general-population-loss} with $\mathcal{A}$ for $T$ timesteps, given initialization $\theta_0$. Then the resultant forward map defined in \cref{eq:def-EGOP-forward-map} is \emph{equivariant} under rotations of the data:
    \[
        f_{U}\left(\theta^{(U)}_t; x\right) = f_{\mathbb{I}_d}\left(\theta^{(\mathbb{I}_d)}_t; U^\T x\right) \quad \forall t\in \{1,\dots,T\}.
    \]
    Equivariance of the forward map immediately implies equivariance of the decision boundary. Moreover, the iterates themselves are \emph{invariant} under rotations of the data distribution: for all $t \in \{1,\dots,T\}$, we have that $\theta^{(U)}_t=\theta^{(\mathbb{I}_d)}_t$.
\end{theorem}
We note that while these results are presented for reparameterization with the exact EGOP eigenbasis, as defined in \cref{eq:EGOP-eigenbasis}, our proofs extend to the setting when the EGOP matrix is estimated with gradient samples.

\subsubsection{Proof of Theorem~\ref{thm:EGOP-invariant-decision-boundaries}}\label{sssec:proof-generalized-equivariance}

To prove \cref{thm:EGOP-invariant-decision-boundaries}, we first prove Lemma~\ref{lemma:generalized-rotations-in-data-are-rotations-in-param}, which generalizes Lemma~\ref{lemma:generalized-rotations-in-data-are-rotations-in-param-gamma}, and establish a few helper lemmas. Note that in the special case covered by Lemma~\ref{lemma:generalized-rotations-in-data-are-rotations-in-param-gamma}, we have that
\begin{equation}\label{eq:def-Q-U-gamma}
    Q^{(\gamma)} \defeq Q^{(U(\gamma))}
\end{equation}
for $Q^{(U)}$ defined below.
\begin{lemma}\label{lemma:generalized-rotations-in-data-are-rotations-in-param}
    Consider any $\mathcal{F}$ satisfying Definition~\ref{def:multilayer-network}, any $f\in \mathcal{F}$ and any $U\in O(d)$. Then there exists a $Q^{(U)} \in O(p)$ such that
    \[
        f(\theta; Ux) = f(Q^{(U)}\theta; x).
    \]
    Specifically, $Q^{(U)}$ can be constructed as
    \begin{equation}\label{eq:def-Q-U}
        Q^{(U)} \defeq \begin{bmatrix}
            U^\T \otimes \mathbb{I}_{m_1} & \vec{0}_{m_1 d}(\vec{0}_{p-m_1 d})^\T \\
            \vec{0}_{p-m_1 d}{(\vec{0}_{m_1 d})}^\T & \mathbb{I}_{p-(m_1 d)} 
        \end{bmatrix}
    \end{equation}
    Here $\vec{0}_k$ denotes the all-zeros vector of dimension $k$.
\end{lemma}
\begin{proof}
    By definition, $h_1(\theta; Ux) = Ux$. Thus
    \begin{align*}
        h_2(\theta; Ux) &= \act_2(W_1 h_1(\theta; Ux) + b_1)\\
        &=\act_2(W_1 Ux + b_1)\\
        &=\act_2\left(\reshape\left((U^\T \otimes \mathbb{I}_{m_1})\vecop(W_1)\right) x + b_1\right)
    \end{align*}
    where the last equality uses the following property of Kronecker products:
    \[
        \reshape\left((A\otimes B)\theta\right) = B \reshape(\theta) A^\T.
    \]
    Consider the block-diagonal matrix $Q^{(U)} \in O(p)$, defined in \cref{eq:def-Q-U}. In what follows, let $Q\defeq Q^{(U)}$ for brevity. Under this transformation, the first $m_1\cdot d$ entries of $Q\theta$ are
    \[
        (Q \theta)_{1:(m_1\cdot d)} = (U^\T \otimes \mathbb{I}_{m_1})\theta_{1:m_1}= (U^\T \otimes \mathbb{I}_{m_1})\vecop(W_1)
    \]
    because $\vecop(W_1) =\theta_{1:m_1}$, and the subsequent entries of $\theta$ (hence, the weights of later layers and all bias vectors) are unchanged:
    \[
        [Q \theta]_{j} = \theta_{j}, \quad \forall j\in [m_1\cdot d + 1, p].
    \]
    Thus
    \[
        \act_2\left(\reshape\left((U^\T \otimes \mathbb{I}_{m_1})\vecop(W_1)\right) x +b_1\right) = h_2(Q\theta; x)
    \]
    and we conclude
    \[
        h_2(\theta; Ux) = h_2(Q\theta; x).
    \]
    Similarly, because $Q$ leaves $W_\ell, b_{\ell}$ invariant for all $\ell \in [2,L+1]$,
    \[
        h_3(\theta; Ux) = \act_3(W_2 h_2(\theta; Ux) + b_2) = \act_3(W_2(h_2(Q\theta;x)) + b_2) = h_3(Q\theta;x).
    \]
    Applying the same argument recursively yields
    \[
        h_\ell(\theta; Ux) = h_\ell(Q\theta; x) \quad \forall \ell \in [2, L+1]
    \]
    and we thus conclude
    \[
        f(\theta; Ux) = W_{L+1} h_L(\theta; Ux) + b_{L+1} = W_{L+1} h_L(Q\theta; x)  + b_{L+1} = f(Q\theta; x)
    \]
    where we've used the fact that $Q$ leaves $W_{L+1}$ invariant.
\end{proof}

\begin{lemma}\label{lemma:gen-rotational-invariance-fact1}
    Assume the expectation in \cref{eq:population_loss} is finite for any fixed $\theta\in \R^p$. Then for any $\theta\in \R^{p}$ and any  $U \in O(d)$, we have that
    \begin{align*}
        \mathbb{E}_{\mathcal{D}^{(U)}}[\ell(y; f(\theta; x))] &= \mathbb{E}_{\mathcal{D}}[\ell(y; f(\theta; Ux))]; \\
        \mathbb{E}_{\mathcal{D}^{(U)}}[\ell(y; f_{U}(\theta; x))] &= \mathbb{E}_{\mathcal{D}}[\ell(y; f_{U}(\theta; Ux))];
    \end{align*}
\end{lemma}
\begin{proof}[Proof of Lemma~\ref{lemma:gen-rotational-invariance-fact1}]
    By the definition of $\mathcal{D}^{(U)}$ in~\cref{eq:def-rotated-dist}, we have that
    $
        \mathcal{D}^{(U)}\left((x, y)\right) = \mathcal{D}\left((U ^\T x, y)\right).
    $
    Thus for any function $g:\R^d\times \R \rightarrow \R$ whose expectation with respect to $\mathcal{D}(\cdot,\cdot)$ is finite, it holds that
    \[
        \mathbb{E}_{ \mathcal{D}^{U}}\left[ g\left((x, y)\right)\right] = \mathbb{E}_{\mathcal{D}}\left[ g\big((U x, y)\big)\right].
    \]
    For settings when the expectation in \cref{eq:population_loss} is finite for any fixed $\theta\in \R^p$, setting $g(x,y) = f(\theta; x)$ in the above equality implies the first result.
    
    Recall that by definition, the EGOP-reparameterized network satisfies
    \[
        f_{U}(\theta; x) \defeq f(V_{U} \theta; x).
    \]
    Under the assumption that the expectation in \cref{eq:population_loss} is finite for any $\theta\in \R^p$, we conclude
    \[
        \mathbb{E}_{\mathcal{D}}[\ell(y; f_{U}(\theta; x))] = \mathbb{E}_{\mathcal{D}}[\ell(y; f(V_{U}\theta; x))]
    \]
    is finite for any $\theta$. Thus, the result for $f_{U}$ follows by an argument analogous to that for $f$.
\end{proof}


\begin{lemma}\label{lemma:gen-rotational-invariance-fact4}
    For any $\rho(\cdot)$ satisfying Assumption~\ref{assumption:rotational-invariance} and any $U\in O(d)$,
    \[
        V_U = \left(Q^{(U)}\right)^\T V_{\mathbb{I}_d}  
    \]
    where $V_U$ is defined in \eqref{eq:EGOP-eigenbasis}, $V_{\mathbb{I}_d}$ is an instantiation of \cref{eq:EGOP-eigenbasis} with the identity transformation $\mathbb{I}_d$,  and $Q^{(U)}$ is defined in \eqref{eq:def-Q-U}.
\end{lemma}
\begin{proof}[Proof of Lemma~\ref{lemma:gen-rotational-invariance-fact4}]
    Consider the map $\mathcal{L}^{(U)}(\cdot; f):\R^{p}
    \rightarrow \R$. Expanding the definition of $\loss^{(U)}$ and applying Lemmas~\ref{lemma:generalized-rotations-in-data-are-rotations-in-param} and \ref{lemma:gen-rotational-invariance-fact1},
    \begin{align*}
        \loss^{(U)} (\theta; f) = \mathbb{E}_{ \mathcal{D}^{(U)}}[\ell(y; f(\theta; x)]
        = \mathbb{E}_{\mathcal{D}}[\ell(y; f(\theta; Ux)]
        = \mathbb{E}_{\mathcal{D}}[\ell(y; f(Q^{(U)}\theta; x)]
        = \loss\left(Q^{U} \theta; f\right).
    \end{align*}
    Thus by chain rule, we conclude
    \[
        \grad_\theta \loss^{(U)} (\theta; f) = \left(Q^{(U)}\right)^\T \grad \loss(Q^{U} \theta; f).
    \]
    We now substitute the above characterization of the gradient into the definition of  $\EGOP\left(\mathcal{L}^{(U)}\right)$ (cf. \cref{eq:egop-def}): 
    \begin{align*}
        &\EGOP\left(\mathcal{L}^{(U)}(\cdot; f)\right)= \mathbb{E}_{\theta \sim \rho}[\nabla_{\theta} \mathcal{L}^{(U)}(\theta; f)\nabla_{\theta} \mathcal{L}^{(U)}(\theta; f)^\T]\\
        &\quad =
        \mathbb{E}_{\theta \sim \rho}
        \left[
            \left(Q^{(U)}\right)^\T \grad \loss(Q^{U} \theta; f) \left(\left(Q^{(U)}\right)^\T \grad \loss(Q^{U} \theta; f)\right)^\T
        \right]\\
        &\quad =
        \left(Q^{(U)}\right)^\T \mathbb{E}_{\theta \sim \rho}
        \left[
             \grad \loss(Q^{U} \theta; f) \left(\grad \loss(Q^{U} \theta; f)\right)^\T
        \right] Q^{(U)}.
    \end{align*}
    Because $\rho(\cdot)$ is rotationally invariant and $Q^{(U)}\in O(p)$, the above implies
    \begin{align*}
        &\EGOP\left(\mathcal{L}^{(U)}\right)= \left(Q^{(U)}\right)^\T \mathbb{E}_{\theta \sim \rho}
        \left[
             \grad \loss(\theta; f) \left(\grad \loss(\theta; f)\right)^\T
        \right] Q^{(U)}\\
        &\quad =\left(Q^{(U)}\right)^\T \EGOP\left(\mathcal{L}^{(\mathbb{I}_d)}\right) Q^{(U)}
    \end{align*}
    because by definition,
    \[
        \mathbb{E}_{\theta \sim \rho}
        \left[
             \grad \loss(\theta; f) \left(\grad \loss(\theta; f)\right)^\T
        \right] = \EGOP(\mathcal{L}(\cdot; f)) =  \EGOP\left(\mathcal{L}^{(\mathbb{I}_d)}(\cdot; f)\right).
    \]
    Thus $\EGOP\left(\mathcal{L}^{(U)}\right)$ and $\EGOP\left(\mathcal{L}^{(\mathbb{I}_d)}\right)$ are equal up to similarity, so their eigenvectors satisfy
    \[
         V_U = \left(Q^{(U)}\right)^\T V_{\mathbb{I}_d}.
    \]
\end{proof}


\begin{lemma}\label{lemma:gen-rotational-invariance-fact5}
    For all $\theta\in \R^{p}$, all $x\in \R^d$, and all $U\in O(d)$,
    \[
        f_U(\theta; x) = f_{\mathbb{I}_d}(\theta; U^\T x).
    \]
\end{lemma}
\begin{proof}[Proof of Lemma~\ref{lemma:gen-rotational-invariance-fact5}]
    We begin by expanding $f_U(\cdot)$ following the definition in \eqref{eq:def-EGOP-forward-map}:
    \begin{align*}
        f_U(\theta; x) = f(V_U \theta; x) = f\left(\left(Q^{(U)}\right)^\T V_{\mathbb{I}_d} \theta; x\right).
    \end{align*}
    By \cref{eq:def-Q-U},
    \[
        \left(Q^{(U)}\right)^\T = \begin{bmatrix}
            U^\T \otimes \mathbb{I}_{m_1} & \vec{0}_{m_1 d}(\vec{0}_{p-m_1 d})^\T \\
            \vec{0}_{p-m_1 d}{(\vec{0}_{m_1 d})}^\T & \mathbb{I}_{p-(m_1 d)} 
        \end{bmatrix}^\T = \begin{bmatrix}
            U \otimes \mathbb{I}_{m_1} & \vec{0}_{p-m_1 d}(\vec{0}_{m_1 d})^\T \\
            \vec{0}_{m_1 d}(\vec{0}_{p-m_1 d})^\T & \mathbb{I}_{p-(m_1 d)} 
        \end{bmatrix}
    \]
    where we've used the fact that $(A\otimes B)^\T = A^\T \otimes B^\T$. Thus
    \begin{align*}
        \left[\left(Q^{(U)}\right)^\T V_{\mathbb{I}_d} \theta\right]_{j} &= \left[V_{\mathbb{I}_d} \theta\right]_{j} \quad \forall j\in [d m_1 + 1, p- dm_1]; \quad \text{and} \\
        \left[\left(Q^{(U)}\right)^\T V_{\mathbb{I}_d} \theta\right]_{1:(d m_1)} &= (U^\T \otimes \mathbb{I}_{m_1})^\T \left[ V_{\mathbb{I}_d} \theta\right]_{1:(d m_1)} = \mathbb{I}_{m_1} \reshape\left(\left[V_{\mathbb{I}_d} \theta\right]_{1:(d m_1)}\right) U^\T,
    \end{align*}
    where the latter follows from the matrix-vector property of Kronecker products $$\reshape\left((A\otimes B)\theta\right) = B \reshape(\theta) A^\T.$$

    Given a parameter vector $z\in \R^p$, let $W_k[z]$ and $b_k[z]$ denote the weights and biases defined by the entries of $z$, following~\cref{def:multilayer-network}. Using this notation, the above implies
    \begin{equation}\label{eq:Q-weights}
        W_k\left[\left(Q^{(U)}\right)^\T V_{\mathbb{I}_d} \theta\right] = \begin{cases}
            W_k\left[ V_{\mathbb{I}_d} \theta\right]U^\T \quad &\text{if $k=1$}; \\
             W_k\left[ V_{\mathbb{I}_d} \theta\right]  &\text{otherwise}.
        \end{cases}
    \end{equation}
    and
    \begin{equation}\label{eq:Q-biases}
        b_k\left[\left(Q^{(U)}\right)^\T V_{\mathbb{I}_d} \theta\right] = b_k\left[V_{\mathbb{I}_d} \theta\right], \;\; \text{for all $k$.}
    \end{equation}

    Recalling the definitions of the layer output functions $h_k$ in~\cref{def:multilayer-network}, this implies 
    \[
        h_1\left(\left(Q^{(U)}\right)^\T V_{\mathbb{I}_d} \theta; x\right) = x
    \]
    and thus, applying \cref{eq:Q-weights} and \cref{eq:Q-biases},
    \begin{align*}
        h_2\left(\left(Q^{(U)}\right)^\T V_{\mathbb{I}_d} \theta; x\right) &= \sigma_2\left( W_1\left[\left(Q^{(U)}\right)^\T V_{\mathbb{I}_d} \theta\right] x +b_1\left[\left(Q^{(U)}\right)^\T V_{\mathbb{I}_d} \theta\right] \right)\\
        &= \sigma_2\left( W_1\left[V_{\mathbb{I}_d} \theta\right]U^\T x +b_1\left[V_{\mathbb{I}_d} \theta\right] \right)\\
        &= h_2\left( V_{\mathbb{I}_d} \theta;U^\T x\right).
    \end{align*}
    Thus by induction, for all $k\in [3, L+1]$,
    \begin{align*}
        h_k\left(\left(Q^{(U)}\right)^\T V_{\mathbb{I}_d} \theta; x\right) &= \sigma_k\Bigg( W_{k-1}\left[\left(Q^{(U)}\right)^\T V_{\mathbb{I}_d} \theta\right] h_{k-1}\left(\left(Q^{(U)}\right)^\T V_{\mathbb{I}_d} \theta; x\right)\\
        &\hspace{2cm} +b_{k-1}\left(\left(Q^{(U)}\right)^\T V_{\mathbb{I}_d} \theta\right) \Bigg)\\
        &= \sigma_k\Bigg( W_{k-1}\left[V_{\mathbb{I}_d} \theta\right] h_{k-1}\left(\left(Q^{(U)}\right)^\T V_{\mathbb{I}_d} \theta; x\right) +b_{k-1}\left[V_{\mathbb{I}_d} \theta\right] \Bigg)\\
        &= \sigma_k\Bigg( W_{k-1}\left[V_{\mathbb{I}_d} \theta\right] h_{k-1}\left(V_{\mathbb{I}_d} \theta; U^\T x\right) +b_{k-1}\left[V_{\mathbb{I}_d} \theta\right] \Bigg)\\
        &=h_k\left(V_{\mathbb{I}_d} \theta; U^\T x\right).
    \end{align*}
    We thus conclude 
    \begin{align*}
        f\left(\left(Q^{(U)}\right)^\T V_{\mathbb{I}_d} \theta; x\right) &= W_{L+1}\left[\left(Q^{(U)}\right)^\T V_{\mathbb{I}_d} \theta\right] h_{L+1}\left(\left(Q^{(U)}\right)^\T V_{\mathbb{I}_d} \theta; x\right)\\
        &\hspace{2cm}+ b_{L+1}\left(\left(Q^{(U)}\right)^\T V_{\mathbb{I}_d} \theta\right)\\
        &= W_{L+1}\left[ V_{\mathbb{I}_d} \theta\right] h_{L+1}\left(\left(Q^{(U)}\right)^\T V_{\mathbb{I}_d} \theta; x\right)+ b_{L+1}\left[ V_{\mathbb{I}_d} \theta\right]\\
        &= W_{L+1}\left[ V_{\mathbb{I}_d} \theta\right] h_{L+1}\left( V_{\mathbb{I}_d} \theta;U^\T x\right)+ b_{L+1}\left[ V_{\mathbb{I}_d} \theta\right]\\
        &=f\left(V_{\mathbb{I}_d} \theta; U^\T x\right).
    \end{align*}
    The result then follows by the definition of the EGOP-reparameterized map in \cref{eq:def-EGOP-forward-map}:
    \[
        f_U(\theta; x) = f\left(\left(Q^{(U)}\right)^\T V_{\mathbb{I}_d} \theta; x\right) = f\left(V_{\mathbb{I}_d} \theta; U^\T x\right) = f_{\mathbb{I}_d}(\theta; x).
    \]
\end{proof}
We are now equipped to prove~\cref{thm:EGOP-invariant-decision-boundaries}. Theorem~\ref{thm:EGOP-invariant-decision-boundaries} is a direct consequence of Lemma~\ref{thm:objective-invariant} below, which implies that for any $U_1, U_2\in O(d)$, $\loss^{(U_1)}(\cdot\,; f_{U_1})$ and $\loss^{(U_2)}(\cdot\,; f_{U_2})$ are identical as functions,
and thus first-order methods produce identical iterates when used to minimize them.
Consequently, the implicit bias of EGOP-reparameterized algorithms is invariant to rotations of the data distribution.

\begin{lemma}\label{thm:objective-invariant}
    Given $\rho$ satisfying Assumption~\ref{assumption:rotational-invariance}, for any $\theta\in \R^{p}$ and any $U \in O(d)$ it holds that
    \[
        \loss^{(U)}(\theta; f_{U}) =
        \loss(\theta; f_{\mathbb{I}_d}).
    \]
    Since this holds at every $\theta$, the (sub)gradients of both functions are also identical:
    \[
        \grad_{\theta}\loss^{(U)}(\theta; f_{U}) = \grad_{\theta}\loss(\theta; f_{\mathbb{I}_d}).
    \]
\end{lemma}

\begin{proof}[Proof of \cref{thm:objective-invariant}]
    By the definition of $\mathcal{L}^{(U)}$ and by Lemma~\ref{lemma:gen-rotational-invariance-fact1},
    \[
        \mathcal{L}^{(U)}(\theta; f_{U}) = \mathbb{E}_{\mathcal{D}^{(U)}} \left[\ell(y; f_U(\theta; x)\right] = \mathbb{E}_{\mathcal{D}} \left[\ell(y; f_U(\theta; Ux)\right].
    \]
    Lemma~\ref{lemma:gen-rotational-invariance-fact5} implies
    \[
        f_U(\theta; Ux) = f_{\mathbb{I}_d}(\theta; U^\T Ux) = f_{\mathbb{I}_d}(\theta; x)
    \]
    where the last equality holds because $U\in O(d)$. Thus
    \[
        \mathbb{E}_{\mathcal{D}} \left[\ell(y; f_U(\theta; Ux)\right] = \mathbb{E}_{\mathcal{D}} \left[\ell(y; f_{\mathbb{I}_d}(\theta; x)\right]
    \]
    so we conclude
    \[
        \mathcal{L}^{(U)}(\theta; f_{U}) = \mathbb{E}_{\mathcal{D}} \left[\ell(y; f_{\mathbb{I}_d}(\theta; x)\right]  = \mathcal{L}(\theta; f_{\mathbb{I}_d}). 
    \]
\end{proof}

\begin{proof}[Proof of \cref{thm:EGOP-invariant-decision-boundaries}]
    To show that $\theta^{(U)}_T=\theta^{(\mathbb{I}_d)}_T$, we proceed by induction. For $T=0$, we have
    \[
        \theta^{(U)}_0=\theta^{(\mathbb{I}_d)}_0 = \theta_0.
    \]
    Thus by \cref{thm:objective-invariant},
    \[
        \nabla_{\theta} \mathcal{L}^{(U)}(\theta^{(U)}_0; f_U)= \nabla_{\theta} \mathcal{L}(\theta^{(\mathbb{I}_d)}_0; f_{\mathbb{I}_d}).
    \]
    Consequently,
    \[
        \theta_1^{(U)} = \update\left(
            \set{\theta_0}, \set{\grad \loss^{(U)}(\theta_0; f_U)}
        \right) = \update\left(
            \set{\theta_0},
            \set{\grad \loss(\theta_0; f_{\mathbb{I}_d})}
        \right) = \theta_1^{(\mathbb{I}_d)}.
    \]
    Assume that $\theta^{(U)}_{\tau} = \theta_{\tau}^{(\mathbb{I}_d)} \equiv \theta_{\tau}$ for
    all $\tau \leq t$. By~\cref{thm:objective-invariant}, that implies
    \[
        \grad_{\theta}\loss^{(U)}(\theta_\tau; f_U) =
        \grad_{\theta}\loss(\theta_{\tau}; f_{\mathbb{I}_d}), \;\; \text{for all $\tau\leq t$}.
    \]
    By definition of the $(t+1)^{\text{st}}$ iterate, we obtain
    \begin{align*}
        \theta^{(U)}_{t+1} &=
        \update\left(
            \set{\theta^{(U)}_\tau}_{\tau\leq t},
            \set{\grad_{\theta}\loss^{(U)}(\theta_\tau; f_U)}_{\tau \leq t}
        \right) \\
        &=
        \update\left(
            \set{\theta^{(\mathbb{I}_d)}_\tau}_{\tau\leq t},
            \set{\grad_{\theta}\loss(\theta_\tau; f_{\mathbb{I}_d})}_{\tau\leq t}
        \right) \\
        &= \theta^{(\mathbb{I}_d)}_{t+1},
    \end{align*}
    whence $\theta^{(U)}_{T} = \theta_{T}^{(\mathbb{I}_d)}$ for all $T \geq 0$.

    To prove the remaining identity, recall that by~\cref{lemma:gen-rotational-invariance-fact5}, for all $\theta\in \R^{2m}, \gamma\in \R$, we have
    \[
        f_U(\theta;x) = f_{\mathbb{I}_d}(\theta; U^\T x).
    \]
    Therefore, because $\theta^{(U)}_{T} = \theta_{T}^{(\mathbb{I}_d)}$ for all $T \geq 0$, we have that
    \[
        f_U\left(\theta^{(U)}_T;x\right) = f_{\mathbb{I}_d}\left(\theta^{(U)}_T; U^\T x\right)= f_{\mathbb{I}_d}\left(\theta^{(\mathbb{I}_d)}_T; U^\T x\right).
    \]
\end{proof}

\subsubsection{Eigenbasis Computation}\label{sssec:eigenmatrix-computation-assumption}

When the EGOP matrix contains repeated eigenvalues, the definition of $V_U$ in \cref{eq:EGOP-eigenbasis} is non-unique. We resolve this ambiguity by making the following assumption: given the similar PSD matrices $\EGOP(\mathcal{L})$ and $Q \EGOP(\mathcal{L})Q^\T$, any procedure used to compute eigenvectors will return
\[
    V \gets \texttt{eigenvectors}(\EGOP(\mathcal{L})) \quad \text{and}\quad (QV) \gets \texttt{eigenvectors}(Q\EGOP(\mathcal{L})Q^\T).
\]
When $\EGOP(\mathcal{L})$ has distinct eigenvalues, the above is guaranteed up to the signs of the columns of $QV$. Moreover, our proofs can be extended to account for the fact that this equality only holds up to sign\footnote{To do so, one need only add the assumption that the algorithm $\mathcal{A}$ is equivariant under negation of entries of $\theta$, a property shared by most optimization algorithms employed in machine learning, including (stochastic) gradient descent, \SignGD, \adam, AdamW, Adagrad, and more recently proposed optimizers like SOAP, Shampoo, and Lion \cite{vyas2024soap,jordan6muon, gupta2018shampoo}. }, and if $\rho$ is assumed to be supported over all of $\R^d$ (e.g., $\rho$ a scaled Gaussian) then this assumption can further be weakened to only hold for eigenvectors corresponding to \textit{non-zero} eigenvalues.

The above assumption is mild because with rich, real-world data, $\EGOP(\mathcal{L})$ will not contain repeated non-zero eigenvalues. Thus this assumption does not limit the applicability of our result in typical settings. 

%% file: src/Shampoo.tex
\section{Equivariance of Structure-Aware Preconditioning Methods}\label{appendix:Shampoo}

In this section, we show that our arguments establishing equivariance under EGOP-reparam\-eterization can be extended to show that related optimization methods, namely structure-aware preconditioning methods including Shampoo and SOAP, can also yield equivariant adaptive optimization algorithms. In \cref{alg:Shampoo}, we present pseudocode instantiating Shampoo for a feed-forward network of the form in Definition~\ref{def:multilayer-network}. Shampoo constructs left- and right-preconditioner matrices $L^{(j)}_t$ and $R^{(j)}_t$ for each layer $j\in \{1,\dots,L+1\}$. We show that these preconditioner matrices lead to equivariance. SOAP, a closely related method, constructs analogous preconditioners and our arguments extend to this algorithms as well. 

\paragraph{Setup} We follow a similar problem setup to that in \cref{sssec:general-equivariance}, which we briefly review. Fix any data distribution $\mathcal{D}$ and any loss function $\ell(\cdot; \cdot)$ such that the population loss $\mathcal{L}$ (cf. \cref{eq:population_loss}) is finite for any fixed $\theta\in \R^p$. Fix any family of networks satisfying Definition~\ref{def:multilayer-network}, any data rotation $U\in O(d)$, any number of time steps $T$, and any initialization $\theta_0 \in \R^{p}$. Recall the definitions of the rotated data distribution $\mathcal{D}^{(U)}$ (cf. \cref{eq:def-rotated-dist}) and the population loss under the rotation, $\mathcal{L}^{(U)}$ (cf. \cref{eq:rotated-general-population-loss}).

Let $\theta_{T}^{(\mathbb{I}_d)}$ denote the result of initializing at $\theta_0$ and training with Shampoo for the original data distribution,  i.e. minimizing loss $\mathcal{L}(\cdot, \cdot)$ for $T$ timesteps. We label these iterates with superscript $\mathbb{I}_d$ because the original data distribution corresponds to the rotation under the identity rotation. Let $\theta^{(U)}_{T}$ denote the analogous parameters obtained by training with gradient descent on the \textit{rotated} data distribution. Specifically, $\theta^{(U)}_{T}$ is the result of initializing at the rotated point $(Q^{(U)})^\T \theta_0$, for $Q^{(U)}$ defined in \cref{eq:def-Q-U}, and minimizing loss $\mathcal{L}^{(U)}(\cdot, \cdot)$ for $T$ time steps with Shampoo. Then the following holds:
\begin{theorem}\label{thm:Shampoo-equvariance-result}
    For any family of networks satisfying Definition~\ref{def:multilayer-network} and any data rotation $U\in O(d)$, the iterates produced by \cref{alg:Shampoo} in the original ($\theta_{t}^{(\mathbb{I}_d)}$) and rotated ($\theta_{t}^{(U)}$) settings are equivalent up to rotation:
    \begin{equation}\label{eq:shampoo-iterates}
        \theta_{t}^{(U)} = (Q^{(U)})^\T \theta_{t}^{(\mathbb{I}_d)} \ \forall t\in \{1,\dots,T\},
    \end{equation}
    where $Q^{(U)}$ is defined in \cref{eq:def-Q-U}. Moreover the resultant trained network forward maps are equivariant:
    \[
        f\left(\theta^{(U)}_{t}; x \right) = f\left(\theta^{(\mathbb{I}_d)}_{t}; U^\T x \right).
    \]
    This implies that the decision boundaries produced by Shampoo are equivariant.
\end{theorem}


In practice, one-sided variants of these algorithms are sometimes employed to reduce computational costs, in which only one of the input- or output-side preconditioner are stored. Our arguments show that in fact to obtain equivariance to data rotations, it suffices to retain only the input-side preconditioners $R^{(j)}_t$. In fact, our proofs further show that this can be weakened, and that input-side reparameterization \textit{of the first layer alone} suffices to convey equivariance with algorithms like Shampoo and SOAP.

While \cref{alg:Shampoo} corresponds to the algorithm analyzed by \citet{gupta2018shampoo}, in practice computing ${L}^{-1/4}_t, {R}^{-1/4}_t$ on every iteration is too computationally intensive, and thus these preconditioners are only updated periodically. These periodic updates are also employed in the closely related SOAP algorithm. Our equivariance results hold under such periodic updates to the preconditioners, \textit{provided the preconditioner matrix roots ${L}^{-1/4}_1, {R}^{-1/4}_1$ are computed on the first iteration.} Thus, in order to achieve equivariance, our results highlight that practitioners should ensure their periodic update scheme includes a matrix root computation on the first iteration. 

\begin{algorithm}[ht]
\caption{\texttt{Shampoo}, instantiated for feed-forward networks satisfying Definition~\ref{def:multilayer-network}}
\begin{algorithmic}[1]
    \State \textbf{Input}: initialization $\theta_0 \in \R^p$, step-size $\eta > 0$, $\epsilon \geq 0$, first-order oracle $\nabla \mathcal{L}(\cdot)$.
    \State \textbf{Initialize} For each layer $j\in \{1,\dots, L+1\}$, initialize $L_0^{(j)} = \epsilon \mathbb{I}_{m_j}$ and $R_0^{(j)} = \epsilon \mathbb{I}_{m_{j-1}}$, where $m_0$ is identified with the dimension of the input data: $m_0 = d$.
    \For{$t = 0, 1, \dots$} 
        \State $G_t \gets \nabla_{\theta} \mathcal{L}(\theta_t) $
        \State $S_0 \gets 0$ \Comment{Index counter}
        \For{$j = 1, \dots, L+1$} \Comment{Update for each layer's weight matrix}
            \State $S_j\gets S_{j-1} + m_{j-1}m_j$ \Comment{Update index counter}
            \State $G_{t}^{(j)} \gets \reshape\left( (G_t)_{S_{j-1}:S_j}\right)$
            \State $L_{t+1}^{(j)} \gets L_{t}^{(j)} + G_t^{(j)}\left(G_t^{(j)}\right)^\T$
            \State $R_{t+1}^{(j)} \gets R_{t}^{(j)} + \left(G_t^{(j)}\right)^\T G_t^{(j)}$
            \State $\left(\theta_{t+1}\right)_{S_{j-1}:S_j} \gets \left(\theta_{t}\right)_{S_{j-1}:S_{j}} -\eta \vecop\left(\left(L_{t+1}^{(j)}\right)^{-1/4} G^{(j)}_t \left(R_{t+1}^{(j)}\right)^{-1/4}\right)$
        \EndFor
    \EndFor
\end{algorithmic}
\label{alg:Shampoo}
\end{algorithm}

\begin{proof}[Proof of Theorem~\ref{thm:Shampoo-equvariance-result}]
    We begin by establishing \cref{eq:shampoo-iterates}. When $t=0$, the equality holds by construction:
    \[
        \theta_{0}^{(U)} = (Q^{(U)})^\T \theta_0 = (Q^{(U)})^\T \theta_{0}^{(\mathbb{I}_d)}.
    \]
    We show that this equality is preserved by the Shampoo update. Consider the map $\mathcal{L}^{(U)}(\cdot; f):\R^{p}
    \rightarrow \R$. Expanding the definition of $\loss^{(U)}$ and applying Lemmas~\ref{lemma:generalized-rotations-in-data-are-rotations-in-param} and \ref{lemma:gen-rotational-invariance-fact1},
    \begin{align*}
        \loss^{(U)} (\theta; f) = \mathbb{E}_{ \mathcal{D}^{(U)}}[\ell(y; f(\theta; x)]
        = \mathbb{E}_{\mathcal{D}}[\ell(y; f(\theta; Ux)]
        = \mathbb{E}_{\mathcal{D}}[\ell(y; f(Q^{(U)}\theta; x)]
        = \loss\left(Q^{U} \theta; f\right).
    \end{align*}
    Thus by chain rule, we conclude that for any $\theta$,
    \begin{equation}\label{eq:shampoo-chain-rule}
        \grad_\theta \loss^{(U)} (\theta; f) = \left(Q^{(U)}\right)^\T \grad \loss(Q^{U} \theta; f).
    \end{equation}
    Recalling the definition of $Q^{(U)}$ in \cref{eq:def-Q-U}, this implies that for any $\theta$, the first $d m_1$ entries of these two gradients are related by an orthogonal transformation:
    \begin{align*}
        \left(\grad_\theta \loss^{(U)} (\theta; f)\right)_{1:(d m_1)} &= \left(U^\T \otimes \mathbb{I}_{m_1}\right)^\T \left(\grad \loss (Q^{(U)}\theta; f)\right)_{1:(d m_1)} \\
        &= \left(U \otimes \mathbb{I}_{m_1}\right) \left(\grad \loss (Q^{(U)}\theta; f)\right)_{1:(d m_1)}
    \end{align*}
    and thus 
    \begin{equation}\label{eq:shampoo-first-layer}
        \reshape\left(\left(\grad_\theta \loss^{(U)} (\theta; f)\right)_{1:(d m_1)}\right) = \reshape\left( \left(\grad \loss (Q^{(U)}\theta; f)\right)_{1:(d m_1)}\right) U^\T
    \end{equation}
    where the last equality uses the following property of Kronecker products:
    \[
        \reshape\left((A\otimes B)\theta\right) = B \reshape(\theta) A^\T.
    \]
    Similarly, by the definition of $Q^{(U)}$, \cref{eq:shampoo-chain-rule} implies that these two gradients are identical in  all later entries, namely for all all $k \in \{d m_1 + 1, \dots, p\}$
    \begin{equation}\label{eq:shampoo-deeper-layers}
        \left(\grad_\theta \loss^{(U)} (\theta; f)\right)_{k} = \left(\grad \loss (Q^{(U)}\theta; f)\right)_{k}.
    \end{equation}
    
    Let $G^{(j)}_t, L^{(j)}_{t+1},$ and $R^{(j)}_{t+1}$ denote the gradient entries and left- and right-side preconditioners constructed by Shampoo when instantiated for the original distribution, i.e., training with oracle access to $\mathcal{L}$. Let $\tilde{G}^{(j)}_t, \tilde{L}^{(j)}_{t+1},$ and $\tilde{R}^{(j)}_{t+1}$ denote analogous quantities for Shampoo instantiated on loss for the rotated distribution, $\mathcal{L}^{(U)}$. On the first iteration,
    \[
        \theta_{0}^{(U)} = (Q^{(U)})^\T \theta_{0}^{(\mathbb{I}_d)} \quad \implies \quad \theta_{0}^{(\mathbb{I}_d)} = Q^{(U)} \theta_{0}^{(U)}
    \]
    so by \cref{eq:shampoo-first-layer} and \cref{eq:shampoo-deeper-layers}, we conclude
    \begin{equation}\label{eq:shampoo-matrix-grads}
        \tilde{G}^{(j)}_0 = \begin{cases}
            G^{(j)}_0 U^\T \quad &j=1\\
            G^{(j)}_0 & j\in \{2,\dots,L+1\}.
        \end{cases}
    \end{equation}
    Thus the left- preconditioners for the first layers satisfy
    \begin{align*}
        \tilde{L}^{(1)}_1 &= \epsilon \mathbb{I}_{m_1} + \tilde{G}^{(j)}_0(\tilde{G}^{(j)}_0)^\T\\
        &= \epsilon \mathbb{I}_{m_1} +(G^{(j)}_0 U^\T)(G^{(j)}_0 U^\T)^\T\\
        &= \epsilon \mathbb{I}_{m_1} +G^{(j)}_0 U^\T U (G^{(j)}_0)^\T\\
        &= \epsilon \mathbb{I}_{m_1} +G^{(j)}_0 (G^{(j)}_0)^\T\\
        &= L^{(1)}_1
    \end{align*}
    where the penultimate inequality holds because $U\in O(d)$. Similarly the right- preconditioners for the first layers satisfy
    \begin{align*}
        \tilde{R}^{(1)}_1 &= \epsilon \mathbb{I}_{d} + (\tilde{G}^{(j)}_0)^\T\tilde{G}^{(j)}_0\\
        &= \epsilon \mathbb{I}_{d} +(G^{(j)}_0 U^\T)^\T (G^{(j)}_0 U^\T)\\
        &= \epsilon \mathbb{I}_{d} +U (G^{(j)}_0)^\T (G^{(j)}_0) U^\T\\
        &= U\left( \epsilon \mathbb{I}_{d} +(G^{(j)}_0)^\T (G^{(j)}_0)\right) U^\T\\
        &= U R^{(1)}_1 U^\T.
    \end{align*}
    Combining these with \cref{eq:shampoo-matrix-grads} implies that
    \begin{equation}\label{eq:shampoo-all-L}
        \tilde{L}^{(j)}_1 = {L}^{(j)}_1 \quad \forall j\in \{1,\dots,L+1\}
    \end{equation}
    and
    \begin{equation}\label{eq:Shampoo-all-R}
        \tilde{R}^{(j)}_1 = \begin{cases}
            U R^{(j)}_1 U^\T \quad &j=1\\
            R^{(j)}_1 & j\in \{2,\dots,L+1\}.
        \end{cases}
    \end{equation}
    We consider the result of the first iteration, using these preconditioners. We first consider $\theta^{(\mathbb{I}_d)}_1$,  the updated parameters when training with the original data distribution. Recall that $\theta^{(\mathbb{I}_d)}_0 = \theta_0$ by construction. Thus for every $j\in \{1, \dots, L+1\}$,
    \begin{align*}
        \left(\theta^{(\mathbb{I}_d)}_1\right)_{S_{j-1}:S_j} = \left(\theta_0\right)_{S_{j-1}:S_{j}} - \eta \vecop\left(\left(L_{1}^{(j)}\right)^{-1/4} G^{(j)}_0 \left(R_{1}^{(j)}\right)^{-1/4}\right).
    \end{align*}
    In the above, $\{S_j\}_{j=1}^{L+1}$ are the index counters constructed in \cref{alg:Shampoo}.

    We now compare with $\theta^{(U)}_1$, the updated parameters when training with the \textit{rotated} data distribution. Recall that $\theta^{(U)}_0 = (Q^{(U)})^\T \theta_0$. By \cref{eq:def-Q-U}, for any $\theta\in \R^p$, $(Q^{(U)})^\T \theta$ leaves invariant all entries of $\theta$ except those corresponding to the first layer weights, i.e., all entries except $(\theta)_{1:dm_1}$. Thus for all $j\in \{2,\dots,L+1\}$,
    \begin{equation}\label{eq:shampoo-inits-deeper-layers}
        \left(\theta^{(U)}_0\right)_{S_{j-1}:S_j} = \left((Q^{(U)})^\T \theta_0\right)_{S_{j-1}:S_j} = \left(\theta_0\right)_{S_{j-1}:S_j}.
    \end{equation}
    Combining this with \cref{eq:shampoo-all-L,eq:Shampoo-all-R,eq:shampoo-matrix-grads} implies that
    \begin{align*}
        \left(\theta^{(U)}_1\right)_{S_{j-1}:S_j} &= \left(\theta^{(U)}_0\right)_{S_{j-1}:S_j} - \eta \vecop\left( \left(\tilde{L}_{1}^{(j)}\right)^{-1/4} \tilde{G}^{(j)}_0 \left(\tilde{R}_{1}^{(j)}\right)^{-1/4}\right)\\
        &= \left(\theta_0\right)_{S_{j-1}:S_j} - \eta \vecop\left(\left({L}_{1}^{(j)}\right)^{-1/4} {G}^{(j)}_0 \left({R}_{1}^{(j)}\right)^{-1/4}\right)\\
        &= \left(\theta^{(\mathbb{I}_d)}_1\right)_{S_{j-1}:S_j}.
    \end{align*}
    We thus find that for all but the first layer, the iterates $\theta^{(U)}_1$ and $\theta^{(\mathbb{I}_d)}_1$ are identical. Because $(Q^{U})^\T \theta$ leaves invariant all entries of $\theta$ except those in the first layer, we also conclude that for all $j\in \{2,\dots,L+1\}$,
    \begin{equation}\label{eq:shampoo-iterate-later-layer-entries}
        \left(\theta^{(U)}_1\right)_{S_{j-1}:S_j} = \left((Q^{(U)})^\T \theta^{(\mathbb{I}_d)}_1\right)_{S_{j-1}:S_j}.
    \end{equation}
    It remains to establish this correspondance for the entries of the first layer. For $j=1$,
    \begin{equation}\label{eq:shampoo-inits-first-layer}
        \left(\theta^{(U)}_0\right)_{S_0:S_1} =\left(\theta^{(U)}_0\right)_{1:(d m_1)} = \left((Q^{(U)})^\T \theta_0\right)_{1:(d m_1)} = (U \otimes \mathbb{I}_{m_1})\left(\theta_0\right)_{1:(d m_1)}
    \end{equation}
    Combining this with \cref{eq:shampoo-all-L,eq:Shampoo-all-R,eq:shampoo-matrix-grads},
    \begin{align*}
        \left(\theta^{(U)}_1\right)_{S_0:S_1} &= \left(\theta^{(U)}_0\right)_{S_0:S_1} - \eta \vecop\left(\left(\tilde{L}_{1}^{(1)}\right)^{-1/4} \tilde{G}^{(1)}_0 \left(\tilde{R}_{1}^{(1)}\right)^{-1/4}\right)\\
         &= (U \otimes \mathbb{I}_{m_1})\left(\theta_0\right)_{S_0:S_1} - \eta \vecop\left(\left({L}_{1}^{(1)}\right)^{-1/4} G^{(1)}_0 U^\T U \left({R}_{1}^{(1)}\right)^{-1/4} U^\T\right)\\
         &= (U \otimes \mathbb{I}_{m_1})\left(\theta_0\right)_{S_0:S_1} - \eta \vecop\left(\left({L}_{1}^{(1)}\right)^{-1/4} G^{(1)}_0 \left({R}_{1}^{(1)}\right)^{-1/4} U^\T\right)\\
         &= (U \otimes \mathbb{I}_{m_1})\left(\theta_0\right)_{S_0:S_1} - \eta (U\otimes \mathbb{I}_{m_1})\vecop\left(\left({L}_{1}^{(1)}\right)^{-1/4} G^{(1)}_0 \left({R}_{1}^{(1)}\right)^{-1/4}\right)\\
         &= (U \otimes \mathbb{I}_{m_1}) \left(\left(\theta_0\right)_{S_0:S_1} - \eta\vecop\left(\left({L}_{1}^{(1)}\right)^{-1/4} G^{(1)}_0 \left({R}_{1}^{(1)}\right)^{-1/4}\right)\right)\\
         &= (U \otimes \mathbb{I}_{m_1})\left(\theta^{(\mathbb{I}_d)}_1\right)_{S_0:S_1}\\
         &= \left((Q^{(U)})^\T \theta^{(\mathbb{I}_d)}_1\right)_{S_0:S_1}.
    \end{align*}
    This, combined with \cref{eq:shampoo-iterate-later-layer-entries}, implies
    \[
        \theta^{(U)}_1 = (Q^{(U)})^\T \theta^{(\mathbb{I}_d)}_1.
    \]
    Applying the above argument in proof-by-induction shows that the equality in \cref{eq:shampoo-iterates} holds for all time steps.

    We now establish the second result. \cref{lemma:generalized-rotations-in-data-are-rotations-in-param} states that for any $f\in \mathcal{F}$ for $\mathcal{F}$ satisfying \cref{def:multilayer-network},
    \[
        f\left(Q^{(U)}\theta^{(\mathbb{I}_d)}_{t}; x \right) = f\left(\theta^{(\mathbb{I}_d)}_{t}; U x \right).
    \]
    Taking $U' \defeq U^\T$ and observing that by definition, $Q^{(U')} = (Q^{(U)})^\T$, \cref{lemma:generalized-rotations-in-data-are-rotations-in-param} also implies
    \[
        f\left((Q^{(U)})^\T \theta^{(\mathbb{I}_d)}_{t}; x \right) = f\left(\theta^{(\mathbb{I}_d)}_{t}; U^\T x \right).
    \]
    Combining the equality in \cref{eq:shampoo-iterates} with this fact yields the equivariance of the forward maps:
    \[
        f\left(\theta^{(U)}_{t}; x \right)=f\left((Q^{(U)})^\T)\theta^{(\mathbb{I}_d)}_{t}; x \right) = f\left(\theta^{(\mathbb{I}_d)}_{t}; U^\T x \right).
    \]
\end{proof}